\documentclass[12pt]{article}

\usepackage{overpic}
\usepackage{rotating}
\usepackage[colorlinks, linkcolor=blue, citecolor=blue]{hyperref}
\usepackage{color}
\usepackage{graphicx,amsmath,amssymb,amsfonts,bm,epsfig,epsf,url,dsfont}
\usepackage{amsthm}
\usepackage{tikz}
\usepackage{bbm}      
\usepackage{booktabs}
\usepackage{cases}
\usepackage{fullpage}

\usepackage[top=1in,bottom=1in,left=1in,right=1in]{geometry}
\usepackage{fancybox}
\usepackage{algorithm}
\usepackage{verbatim}
\usepackage{algorithmic,mathrsfs}

\usepackage{mathtools}

\usepackage{scalerel,stackengine}
\stackMath
\newcommand\reallywidehat[1]{%
\savestack{\tmpbox}{\stretchto{%
  \scaleto{%
    \scalerel*[\widthof{\ensuremath{#1}}]{\kern-.6pt\bigwedge\kern-.6pt}%
    {\rule[-\textheight/2]{1ex}{\textheight}}%WIDTH-LIMITED BIG WEDGE
  }{\textheight}% 
}{0.5ex}}%
\stackon[1pt]{#1}{\tmpbox}%
}

\newcommand{\ceil}[1]{\left \lceil #1 \right \rceil}
\newcommand{\bE}{\mathbb{E}}

\newtheorem{definition}{Definition}
\newtheorem{example}{Example}

\newtheorem{theorem}{Theorem}

\newtheorem{lemma}{Lemma}

\newtheorem{rmk}{Remark}

\newenvironment{fminipage}%
  {\begin{Sbox}\begin{minipage}}%
  {\end{minipage}\end{Sbox}\fbox{\TheSbox}}

\makeatletter
\newcommand*{\rom}[1]{\expandafter\@slowromancap\romannumeral #1@}
\makeatother

\newcommand{\Ind}{\mathbbm{1}}

\newcommand{\vol}{\s{Vol}}

\newcommand{\abs}[1]{\left|#1\right|}
\newcommand{\R}{\mathbb{R}} 
\newcommand{\N}{\mathbb{N}}

\newcommand{\E}{\mathbb{E}}

\def\P{{\mathbb P}}
\def\S{{\mathbb S}}

\newcommand {\pr} {\mathbb{P}}

\newcommand{\calA}{{\cal A}}
\newcommand{\calB}{{\cal B}}

\newcommand{\calD}{{\cal D}}
\newcommand{\calE}{{\cal E}}

\newcommand{\calG}{{\cal G}}
\newcommand{\calH}{{\cal H}}

\newcommand{\calK}{{\cal K}}
\newcommand{\calL}{{\cal L}}
\newcommand{\calM}{{\cal M}}
\newcommand{\calN}{{\cal N}}
\newcommand{\calO}{{\cal O}}
\newcommand{\calP}{{\cal P}}
\newcommand{\calQ}{{\cal Q}}

\newcommand{\calX}{{\cal X}}
\newcommand{\calY}{{\cal Y}}

\newcommand{\be}{\begin{equation}}
\newcommand{\ee}{\end{equation}}
\newcommand{\beqna}{\begin{eqnarray}}
\newcommand{\eeqna}{\end{eqnarray}}

\DeclarePairedDelimiterX{\set}[1]{\{}{\}}{\setargs{#1}}
\DeclarePairedDelimiterX{\cond}[1]{[}{]}{\setargs{#1}}
\NewDocumentCommand{\setargs}{>{\SplitArgument{1}{;}}m}
{\setargsaux#1}
\NewDocumentCommand{\setargsaux}{mm}
{\IfNoValueTF{#2}{#1} {#1\,\delimsize|\,\mathopen{}#2}}%{#1\:;\:#2}

\DeclarePairedDelimiter\parenv{\lparen}{\rparen}

\newcommand{\indep}{\perp \!\!\! \perp}

\newcommand{\p}[1]{\left(#1\right)}
\newcommand{\pp}[1]{\left[#1\right]}
\newcommand{\ppp}[1]{\left\{#1\right\}}
\newcommand{\norm}[1]{\left\|#1\right\|}

\usepackage{xspace}

\usepackage{caption,subcaption}

\newcommand{\s}[1]{\mathsf{#1}}

\allowdisplaybreaks

\makeatletter
\def\thanks#1{\protected@xdef\@thanks{\@thanks
        \protect\footnotetext{#1}}}
\makeatother

\begin{document}
%\title{Detecting Statistically Dependent Databases}
\title{Testing Dependency of Unlabeled Databases}
\author{Vered~Paslev~~~~~~~~~~~~~Wasim Huleihel\thanks{V. Paslev and  W. Huleihel are with the Department of Electrical Engineering-Systems at Tel Aviv University, {T}el {A}viv 6997801, Israel (e-mails:  \texttt{veredpaslev@mail.tau.ac.il, wasimh@tauex.tau.ac.il}). This work is supported by the ISRAEL SCIENCE FOUNDATION (grant No. 1734/21).}}

\maketitle
\sloppy

\begin{abstract}
In this paper, we investigate the problem of deciding whether two random databases $\s{X}\in\calX^{n\times d}$ and $\s{Y}\in\calY^{n\times d}$ are statistically dependent or not. This is formulated as a hypothesis testing problem, where under the null hypothesis, these two databases are statistically independent, while under the alternative, there exists an unknown row permutation $\sigma$, such that $\s{X}$ and $\s{Y}^\sigma$, a permuted version of $\s{Y}$, are statistically dependent with some known joint distribution, but have the same marginal distributions as the null. We characterize the thresholds at which optimal testing is information-theoretically impossible and possible, as a function of $n$, $d$, and some spectral properties of the generative distributions of the datasets. For example, we prove that if a certain function of the eigenvalues of the likelihood function and $d$, is below a certain threshold, as $d\to\infty$, then weak detection (performing slightly better than random guessing) is statistically impossible, no matter what the value of $n$ is. This mimics the performance of an efficient test that thresholds a centered version of the log-likelihood function of the observed matrices. We also analyze the case where $d$ is fixed, for which we derive strong (vanishing error) and weak detection lower and upper bounds.
\end{abstract}

\section{Introduction}\label{sec:intro}

Prompted by practical scenarios, such as computational biology applications \cite{pmid:18725631,kang2012fast}, social network analysis \cite{4531148,5207644}, computer vision \cite{Berg2005,10.5555/2976456.2976496}, and data anonymization/privacy-focused systems, there has been a recent focus on exploring the theoretical underpinnings and algorithmic solutions for database alignment under statistical frameworks. Indeed, quantifying relationships between disparate databases stand as fundamental undertakings in statistics. Modern databases present many challenges: they are high-dimensional, lack labels, contain noise, and appear scrambled. One concrete example of an inference problem, involving a pair of databases, is framed as the following hypothesis testing problem. Under the null hypothesis, no statistical dependency exists between the databases, whereas under the alternative hypothesis, there exists a permutation that reorganizes one database in such a way that the two become dependent. Then, given these databases, under what conditions can we discern whether they are dependent or not?

As a tangible folklore example, consider the following scenario: envision two distinct data sources, such as Netflix and IMDb, each providing feature lists for a set of entities, like users. These features encompass diverse user attributes, such as names, user IDs, and ratings. Frequently, feature labels are either unavailable or intentionally removed to safeguard sensitive personally identifiable information. Consequently, straightforwardly matching feature pairs between the two sources, corresponding to the same user, becomes challenging. Nevertheless, there is optimism that when a substantial correlation exists between the two databases, it becomes feasible to establish connections between them and create a coherent alignment for their respective feature lists \cite{4531148,5207644}.

Recently, there has been a focus on what is known as the \emph{data alignment problem}. This problem can be seen as a straightforward probabilistic model that encapsulates the scenario described above. It was introduced and explored in, e.g., \cite{10.1109/ISIT.2018.8437908,pmlr-v89-dai19b}. In essence, this problem involves two databases, denoted as $\s{X}\in\mathbb{R}^{n\times d}$ and $\s{Y}^{n\times d}$, each comprising $n$ users, each with $d$ features. The key challenge lies in uncovering an unknown permutation or correspondence that matches users in $\s{X}$ with those in $\s{Y}$. When a pair of entries from these databases is matched, their features are dependent according to a known distribution, whereas for unmatched entries, the features are independent. The primary objective is to \emph{recover} the unknown permutation and establish statistical assurances regarding the feasibility and impossibility of this recovery. The feasibility of recovery is contingent upon factors such as the correlation level, $n$, and $d$. The statistical limits of this recovery problem are understood for some specific probability distributions that generate the databases. For instance, in the Gaussian case, where the two databases have independent standard normal entries, with correlation coefficient $\rho$ between the entries of matched rows, it has been demonstrated in \cite{pmlr-v89-dai19b} that perfect recovery is attainable if $\rho^2 = 1-o(n^{-4/d})$, while it becomes impossible if $\rho^2 = 1-\omega(n^{-4/d})$ as both $n$ and $d$ tend to infinity.

The \emph{detection} counterpart of the recovery problem discussed above has also undergone extensive investigation in \cite{9834731,nazer2022detecting,tamir,HuleihelElimelech}, in the Gaussian case. As mentioned earlier, the central question here revolves around determining the correlation level needed to decide whether the two databases are correlated or not. It has been established in \cite{nazer2022detecting} that when $\rho^2d\to\infty$, efficient detection becomes feasible (with an exceedingly small error probability) using a straightforward thresholding of the sum of entries in $\s{X}^T\s{Y}$. Conversely, it has also been demonstrated that when $\rho^2d\sqrt{n}\to0$ and $d=\Omega(\log n)$, detection becomes information-theoretically impossible. Most recently, in \cite{HuleihelElimelech}, the gap between the lower and upper bounds above has been conclusively addressed, by proving that the upper bound is tight, i.e., if $\rho^2d\to0$ as $d\to\infty$, weak detection (performing slightly better than random guessing) is information-theoretically impossible, regardless of the specific value of $n$. In fact, while the main focus in related literature was exclusively on the asymptotic regime where both $n$ and $d$ tend to infinity, \cite{HuleihelElimelech} determine sharp thresholds for all possible asymptotic regimes of $n$ and $d$.

In this paper, we continue the investigation of the database alignment detection problem. While the results above are neat and interesting, they apply for the Gaussian case only. This is in fact true not just for the detection but also for the recovery problem. This case is, however, restricted in the sense that for jointly Gaussian random variables, the notion of uncorrelatedness and independence are equivalent; the information-theoretic limits depend on the distribution through the correlation parameter solely. Accordingly, in the general case, it is a priori unclear how these thresholds depend on the distributions at hand. Furthermore, as it turns out, many of the techniques in \cite{9834731,nazer2022detecting,tamir,HuleihelElimelech} are tailored to the Gaussian case, and the analysis of general distributions becomes challenging. To the best of our knowledge, our paper is the first to provide characterization for the thresholds at which optimal testing is information-theoretically impossible and possible, for two general distributions. Specifically, given two generative distributions, we prove that weak (strong) detection is information-theoretically impossible if a certain function of the eigenvalues of the likelihood function kernel (as defined in \eqref{eqn:kernel}) is below some threshold; we consider all possible asymptotic regimes as a function of $n$ and $d$. We then complement out lower bounds by algorithmic upper bounds, and show that these are tight in a few classical examples. To that end, we propose three detection algorithms and analyze their associated risks and sample complexities. 

We will now provide a brief overview of other related works. In \cite{9174507}, the problem of partial recovery of the hidden alignment was investigated. In \cite{ShiraniISIT}, necessary and sufficient conditions for successful recovery through a typicality-based framework were established. Additionally, \cite{Bakirtas2021DatabaseMU} and \cite{Bakirtas2022DatabaseMU} explored the problem of alignment recovery in cases involving feature deletions and repetitions, respectively. A recent development in this area involved the joint analysis of correlation detection and alignment in Gaussian databases, as presented in \cite{tamir}. Finally, it is worth noting that the challenges in database alignment and detection closely relate to various planted matching problems, specifically the graph alignment problem, which has yielded numerous intriguing results and valuable mathematical techniques. Roughly speaking, in this problem the goal is to detect the edge correlation between two random graphs with unlabeled nodes. For further exploration, one can refer to works such as \cite{GraphAl1,GraphAl2,GraphAl3,GraphAl4,wu2020testing,GraphAl5,GraphAl6,GraphAl7}, and their associated references. In particular, as we mention later on in the outlook of our paper, an interesting open problem mentioned in \cite{wu2020testing} is to investigate the case of general edge weight distributions, which inspired our work for the random databases case.

%We now mention other related work briefly. In \cite{9174507} the problem of partial recovery of the permutation aligning the databases was analyzed. In \cite{ShiraniISIT} necessary and sufficient conditions for successful recovery matching using a typicality-based framework were established. Furthermore, \cite{Bakirtas2021DatabaseMU} and \cite{Bakirtas2022DatabaseMU} proposed and explored the problem of permutation recovery under feature deletions and repetitions, respectively. Recently, the problem of joint correlation detection and alignment of {G}aussian database was analyzed in \cite{tamir}. Recently, the Gaussian setting was analyzed in \cite{HuleihelElimelech}, where several open questions were resolved, most notably, it was shown that if $\rho^2d\to0$, as $d\to\infty$, then weak detection (performing slightly better than random guessing) is statistically impossible, \emph{irrespectively} of the value of $n$. Finally, we note that the problem of database alignment and detection is closely related to a wide verity of planted matching problems, specifically, the graph alignment problem, with many exciting and interesting results, and useful mathematical techniques, see, e.g., \cite{GraphAl1,GraphAl2,GraphAl3,GraphAl4,wu2020testing,GraphAl5,GraphAl6,GraphAl7}, and many references therein.

\paragraph{Notation.}
For any $n\in\mathbb{N}$, the set of integers $\{1,2,\dots,n\}$ is denoted by $[n]$, and $a_1^n = \{a_1,a_2,\ldots,a_n\}$. Let $\mathbb{S}_n$ denotes the set of all permutations on $[n]$. For a given permutation $\sigma\in\mathbb{S}_n$, let $\sigma_i$ denote the value to which $\sigma$ maps $i\in[n]$. Random vectors are denoted by capital letters such as $X$ with transpose $X^T$. A collection of $n$ random vectors is written as $\s{X}$ = $(X_1,\ldots,X_n)$. The notation $(X_1,\ldots,X_n)\sim P_X^{\otimes n}$ means that the random vectors $(X_1,\ldots,X_n)$ are independent and identically distributed (i.i.d.) according to $P_X$. We use $\calN(\eta,\Sigma)$ to represent the multivariate normal distribution with mean vector $\eta$ and covariance matrix $\Sigma$. Let $\s{Poisson}(\lambda)$ denote the Poisson distribution with parameter $\lambda$. The $n\times n$ identity matrix is denoted by $I_{n\times n}$. For probability measures $\mathbb{P}$ and $\mathbb{Q}$, let $d_{\s{TV}}(\mathbb{P},\mathbb{Q})=\frac{1}{2}\intop |\mathrm{d}\mathbb{P}-\mathrm{d}\mathbb{Q}|$ and $d_{\s{KL}}(\mathbb{P}||\mathbb{Q}) = \bE_{\mathbb{P}}\log\frac{\mathrm{d}\mathbb{P}}{\mathrm{d}\mathbb{Q}}$ denote the total variation distance and and the Kullback-Leibler (KL) divergence, respectively. We also define the symmetric KL divergence by $d_{\s{SKL}}(\mathbb{P}||\mathbb{Q}) = \frac{1}{2}\cdot[d_{\s{KL}}(\mathbb{P}||\mathbb{Q})+d_{\s{KL}}(\mathbb{Q}||\mathbb{P})]$. For a probability measure $\mu$ on a space $\Omega$, we use $\mu^{\otimes d}$ for the product measure of $\mu$ ($d$ times) on the product space $\Omega^d$. For a measure $\nu\ll\mu$ (that is, a measure absolutely continuous with respect to $\mu$), we denote (by abuse of notation) the Randon-Nikodym derivative $\nu$ with respect to $\mu$ by $\frac{\nu}{\mu}$. For functions $f,g:\N \to \R$, we say that $f=O(g)$ (and  $f=\Omega(g)$) if there exists $c>0$ such that $f(n)\leq cg(n)$ (and $f(n)\geq cg(n)$) for all $n$. We say that $f=o(g)$ if $\lim_{n\to\infty}f(n)/g(n)=0$, and that $f=\omega(g)$ if $g=o(f)$.

\section{Setup and Learning Problem}\label{sec:model}

\paragraph{Setup.} As mentioned in the introduction, in this paper we will investigate the following decision problem. We deal with two databases $\s{X} = (X_1,\ldots,X_n)$ and $\s{Y}=(Y_1,\ldots,Y_n)$, where $X_i\in\calX^d$ and $Y_i\in\calY^d$, for $i=1,2,\ldots,n$, with $n$ being the number of entities (say, users), and $d$ is the number of features. Now, under the null hypothesis $\calH_0$, the databases $\s{X}$ and $\s{Y}$ are generated \emph{independently} at random, where $X_1,\ldots,X_n\sim P_X^{\otimes d}$ and $Y_1,\ldots,Y_n\sim P_Y^{\otimes d}$, with $P_X = P_Y$ (and therefore $\calX = \calY$ as well). We denote by $\mathbb{P}_0$ the null distribution, i.e., the joint probability distribution of $(\s{X},\s{Y})$ under $\calH_0$. Also, for simplicity of notation, we denote the product measure $Q_{XY}\triangleq P_X\times P_Y$. Under the alternate hypothesis $\calH_1$, the databases $\s{X}$ and $\s{Y}$ are dependent under some \emph{unknown} alignment/permutation $\sigma\in\mathbb{S}_n$, specifically, given $\sigma\in\mathbb{S}_n$, a permutation over $[n]$, we have $(X_1,Y_{\sigma_1}),(X_2,Y_{\sigma_2}),\ldots,(X_n,Y_{\sigma_n})\stackrel{\mathrm{i.i.d.}}{\sim}P_{XY}^{\otimes d}$, with the same marginals $P_X=P_Y$. For a fixed $\sigma\in \S_n$, we denote the joint probability distribution of $(\s{X},\s{Y})$  under the hypothesis $
\calH_{1}$ by $\P_{\calH_1\vert\sigma}$. To summarize, the hypothesis testing problem we deal with is,
\begin{equation}\label{eqn:compositeTesting}
\begin{aligned}
    &\calH_0: (X_1,Y_1),\ldots,(X_n,Y_n)\stackrel{\mathrm{i.i.d}}{\sim} P_X^{\otimes d}\times P_Y^{\otimes d}\\
& \calH_1: (X_1,Y_{\sigma_1}),\ldots,(X_n,Y_{\sigma_n})\stackrel{\mathrm{i.i.d}}{\sim} P_{XY}^{\otimes d},\ \ \text{for some } \sigma\in\mathbb{S}_n.
\end{aligned}
\end{equation}
Viewing $\s{X}$ as an $n\times d$ matrix, we denote its $(i,j)$ element by $X_{ij}$ (and similarly for $\s{Y}$).  We remark here that the distributions $Q_{XY}$ and $P_{XY}$ are allowed to be functions of $n$ and $d$. In fact, as we will see later on, this is when the problem becomes interesting and non-trivial. Fig.~\ref{fig:comp} presents an illustration of our probabilistic model under $\calH_1$.

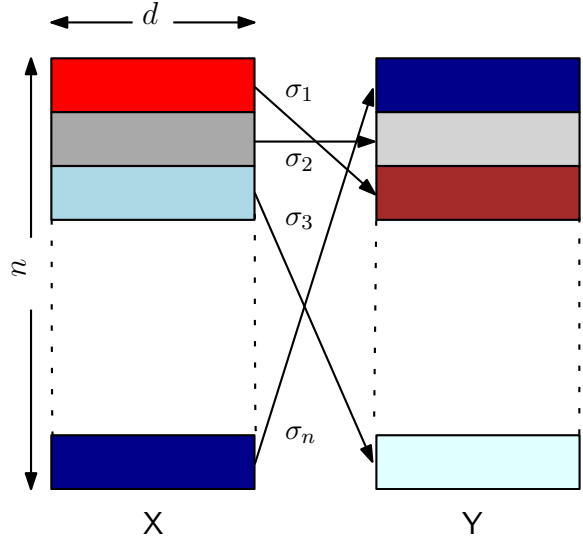
\begin{figure}[t]
\centering
\begin{overpic}[scale=2]
    {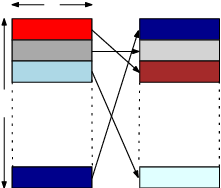}
    \put(-1.5,37.7){\begin{turn}{90}$n$\end{turn}}
    \put(21.6,82.8){$d$}
    \put(21.8,-8){$\s{X}$}
    \put(78.5,-8){$\s{Y}$}
    %\put(64.2,2){$\s{X}_{\calH_0}$}
    %\put(90.5,2){$\s{Y}_{\calH_0}$}
    %\put(90.5,2){$\s{Y}_{\calH_0}$} 
    \put(47,69.8){\small{$\sigma_1$}}
    \put(47,57.4){\small{$\sigma_2$}} 
    \put(47,47){\small{$\sigma_3$}}     
    \put(47,9){\small{$\sigma_n$}}    
\end{overpic}
\vspace*{5mm}
\caption{
An illustration of the probabilistic model under the alternative hypothesis. It can be seen that the databases $\s{X}$ and $\s{Y}$ dependent vectors are marked with a similar color. 
}
\label{fig:comp}
\end{figure}

\paragraph{Learning problem.} Given the databases $\s{X}$ and $\s{Y}$, a test/detection algorithm $\phi_{n,d}:\calX^{n\times d}\times\calY^{n\times d}\to \{0,1\}$, for the hypothesis testing problem above, is tasked with outputting a decision in $\{0,1\}$. We define the risk of a detection algorithm $\phi_{n,d}$ as the sum of its Type-I and (worst-case) Type-II error probabilities, namely,
\begin{align}
\s{R}(\phi_{n,d})\triangleq \P_{\calH_0}[\phi_{n,d}(\s{X},\s{Y})=1]+\max_{\sigma\in \S_n}\P_{\calH_1\vert\sigma}[\phi_{n,d}(\s{X},\s{Y})=0].\label{eqn:worstCaseRisk}
\end{align}
The \textit{minimax} risk is then defined as,
\begin{align}
\s{R}^\star\triangleq\inf_{\phi_{n,d}:\calX^{n\times d}\times\calY^{n\times d}\to \{0,1\}}\s{R}(\phi_{n,d}).\label{eqn:optimalworstCaseRisk}
\end{align}
We remark that $\s{R}$ and $\s{R}^\star$ are in general functions of $n$, $d$, $Q_{XY}$, and $P_{XY}$. However, we omit the dependency on these parameters from our notations for the benefit of readability, as we shall do for our detection algorithms as well.

As in \cite{HuleihelElimelech}, we study the information-theoretic limits (i.e., impossibility and possibility lower and upper bounds) of the hypothesis testing problem in \eqref{eqn:compositeTesting}, for several possible asymptotic regimes of $n$ and $d$. To be more precise, let $\calD$ denote the pair of distributions $Q_{XY}$, and $P_{XY}$. Then, the regimes we investigate are characterized by sequences of the parameters $(\calD,d,n)=(\calD_\ell,d_\ell,n_\ell)_{\ell\in \N}$. For example, if $P_X=P_Y = \calN(0,1)$ and $P_{XY}$ denote the joint distribution of two standard normal random variables with correlation coefficient $\rho$, then the triplet $(\calD,d,n)$ reduces to $(\rho,d,n)$. In this paper, the asymptotic regimes we consider correspond to the scenarios where $d_\ell$ and $n_\ell$ are either fixed or tend to infinity. Accordingly, asymptotic notations, such as, $f(\calD) = o(\cdot)$, $f(\calD)=\Omega(\cdot)$, etc., where $f(\cdot)$ is some one-dimensional function of $\calD$, should be interpreted in terms of the sequences above. For example, the condition $f(\calD)=o(d^{-1})$ means that the sequence $(\calD,d,n)$ satisfies $f(\calD_\ell) d_\ell\to0$, as $\ell\to\infty$. Later on, we will give concrete examples which elucidate the notations above. With the above in mind, we are in a position to define the notions of strong and weak detection.

\begin{definition}
A sequence $(\calD,d,n)=(\calD_k,d_k,n_k)_k$ is said to be:
\begin{enumerate}
    \item Admissible for strong detection if 
    % there exists an algorithm $\phi$ such that 
    $\lim_{k\to \infty}\s{R}^\star=0$.
    \item Admissible for weak detection if
    % if there exists an algorithm $\phi$ such that 
    $\limsup_{k\to \infty} \s{R}^\star<1$.
\end{enumerate}
\end{definition}
A few comments are in order. First, it should be clear that admissibility of strong detection implies the admissibility of weak detection. More concretely, if a test $\phi$ achieves strong detection, i.e., $\s{R}(\phi)\to0$, then $\phi$ achieves weak detection as well. Also, note that strong detection requires the test statistic to determine with high probability whether $(\s{X},\s{Y})$ is is drawn from $\calH_0$ or $\calH_1$, while weak detection only aims at strictly outperforming random
guessing of the underlying hypothesis. Finally, let $\bar{\s{R}}^\star$ denote the optimal \emph{average} risk, which corresponds to the case where $\sigma$ is uniformly drawn over $\mathbb{S}_n$ (rather being unknown). Mathematically, it is defined similarly to \eqref{eqn:worstCaseRisk}--\eqref{eqn:optimalworstCaseRisk}, but with the maximum in \eqref{eqn:worstCaseRisk} over $\sigma\in\mathbb{S}_n$ replaced by an expectation with respect to $\sigma\sim\s{Unif}(\mathbb{S}_n)$. It is clear that $\bar{\s{R}}^\star$ lower bounds the worst-case risk $\s{R}^\star$, and it is well-known that $\bar{\s{R}}^\star=1-d_{\s{TV}}(\pr_{\calH_0},\pr_{\calH_1})$, where with some abuse of notation, in the average case, the distribution $\pr_{\calH_1}$  is defined as the mixture $\pr_{\calH_1}\triangleq\bE_{\sigma}\pr_{\calH_1\vert\sigma}$. Also, the optimal test that achieves the minimal average risk is the likelihood ratio test (a.k.a. the Neyman-Pearson test). Accordingly, weak and strong detection are equivalent to $\liminf d_{\s{TV}}(\pr_{\calH_0},\pr_{\calH_1})>0$ and $\lim d_{\s{TV}}(\pr_{\calH_0},\pr_{\calH_1})=0$, respectively. We discuss this in detail in Section~\ref{sec:proofLowerBounds}. In the same vein, to rule out the possibility of weak/strong detection, we will use the well-known facts that,
\begin{equation}\label{eq:TV1}
    d_{\s{TV}}(\P_{\calH_0},\P_{\calH_1})=o(1)\implies \lim_{k\to\infty}\s{R}^\star=1,
\end{equation}
and
\begin{equation}\label{eq:TV2}
    d_{\s{TV}}(\P_{\calH_0},\P_{\calH_1})\leq1-\Omega(1) \implies\liminf_{k\to\infty} \s{R}^\star>0.
\end{equation}
To wit, the implication in \eqref{eq:TV1} correspond to the impossibility of weak detection, while \eqref{eq:TV2} corresponds to the impossibility of strong detection, respectively. We discuss further and prove these relations in Section~\ref{sec:proofLowerBounds}.

\section{Main Results}\label{sec:main}

In this section, we present our main results on the impossibility and admissibility of weak and strong criteria of detection, under the setting presented in Section~\ref{sec:model}. We differentiate between the various possible asymptotic regimes: (\texttt{I}) both $n,d\to\infty$; the regime analyzed in most related past literature, (\texttt{II}) $n$ is a constant and $d\to\infty$, and (\texttt{III}) $d$ is a constant and $n\to\infty$. We begin with our impossibility lower-bounds. To that end, we introduce a few important notations. 

For any $x\in\calX$ and $y\in\calY$, we let $\calL(x,y) \triangleq \frac{P_{XY}(x,y)}{Q_{XY}(x,y)}$ denote the single-letter likelihood function. It turns our beneficial to think of $\calL$ as a kernel, which induces an operator as follows. For any square-integrable function $f$ under $Q_{XY}$, i.e., $ \int\int f^2(x,y)Q_{X}(\mathrm{d}x)Q_{Y}(\mathrm{d}y)<\infty$,
\begin{align}
    (\calL f)(x) \triangleq \bE_{Y\sim Q_Y}\pp{\calL(x,Y)f(Y)}.\label{eqn:kernel}
\end{align}
Furthermore, the composition $\calL^2= \calL\circ\calL$ is given by $\calL^2(x,y) = \bE_{Z\sim Q}[\calL(x,Z)\calL(Z,y)]$, and $\calL^k$ is defined similarly. We assume that $\calL(x,y) = \calL(y,x)$, and hence $\calL$ is self-adjoint. Furthermore, if we assume that $ \int\int \calL^2(x,y)Q_{X}(\mathrm{d}x)Q_{Y}(\mathrm{d}y)<\infty$, then $\calL$ is Hilbert-Schmidt. Thus, $\calL$ is diagonazable, and we denote its eigenvalues by $\{\lambda_i\}_{i\geq0}$. Accordingly, the trace of $\calL$ is given by $\s{trace}(\calL) = \bE_{Y\sim Q_{Y}}[\calL(Y,Y)] = \sum_{i\in\mathbb{N}}\lambda_i$. Without loss of generality, we assume that the sequence of eigenvalues $\{\lambda_i\}_{i\geq0}$ is decreasing, namely, $\lambda_i\geq\lambda_{i+1}$, for all $i\in\mathbb{N}$. Finally, we prove in Lemma~\ref{lem:eigen} in Section~\ref{sec:proofLowerBounds} that the largest eigenvalue of $\calL$ is unity, i.e., $\lambda_0=1$. We are now in a position to state our main results.
\begin{theorem}[Weak detection lower bound]\label{th:LB}
 Weak detection is impossible as long as 
 \begin{align}
\sum_{i\geq1}\frac{\lambda_i^2}{1-\lambda_i^2}=o(d^{-1}).\label{eqn:weakdetectlowerCond}     
 \end{align}
That is, for a sequence $(\calD,d,n)=(\calD_k,d_k,n_k)_k$ such that \eqref{eqn:weakdetectlowerCond} holds:
\begin{itemize}
     \item If $d$ is any function of $k$, and $n\to \infty$ then $\lim_{k\to \infty}\s{R}^\star =1$.
    \item If $n$ is constant and $d\to \infty$ then $\lim_{k\to \infty}\s{R}^\star =1$. 
\end{itemize}
Namely, $(\calD,d,n)$ is not admissible for  weak detection.
\end{theorem}
Based on Theorem~\ref{th:LB}, we see that weak detection is impossible when \eqref{eqn:weakdetectlowerCond} holds, for all different asymptotic regimes of $n$ and $d$; however, if $d$ is fixed then the right-hand-side of \eqref{eqn:weakdetectlowerCond} should be understood as $o(1)$, as $n\to\infty$. Next, we move forward to our strong detection lower bounds. 
\begin{theorem}[Strong detection lower bounds]\label{th:LB2}
A sequence $(\calD,d,n)$ is not admissible for strong detection if: 
\begin{enumerate}
    \item $d\in \N$ and $\{\lambda_i\}_{i\geq1}$ are constants, such that
    \begin{align}
        d<-\frac{\log\lambda_1^2}{\log\sum_{i\in\mathbb{N}}\lambda_i^{2}},\label{eqn:strongdetectlowerCond1}  
    \end{align}
    and $n\to\infty$.
    \item $n,d\to \infty$, and
     \begin{align}
\sum_{i\geq1}\frac{\lambda_i^2}{1-\lambda_i^2}<(1-\varepsilon)d^{-1},\label{eqn:strongdetectlowerCond2}     
 \end{align}
 for some $\varepsilon>0$, independent of $n$ and $d$.
    \item $d\to\infty$, $n\in\mathbb{N}$, and
     \begin{align}
\sum_{i\geq1}\frac{\lambda_i^2}{1-\lambda_i^2}=O(d^{-1}).\label{eqn:strongdetectlowerCond3}     
 \end{align}
\end{enumerate} 
\end{theorem}
It is evident that for strong detection, we can prove slightly stronger results; for example, when $n,d\to\infty$, comparing \eqref{eqn:weakdetectlowerCond} and \eqref{eqn:strongdetectlowerCond2} we see that the barriers are the same up to a constant factor. We also note that some of the barriers in Theorems~\ref{th:LB} and \ref{th:LB2} can be represented in terms of the likelihood kernel. For example, it is easy to check that the condition in \eqref{eqn:weakdetectlowerCond} can be written as,
\begin{align}
 \s{trace}\p{\calL^2\cdot(\mathbf{I}-\calL^2)^{-1}} = o(d^{-1}),
\end{align}
while the condition in \eqref{eqn:strongdetectlowerCond1} can be also written as,
\begin{align}
    d<-\frac{\log\lambda_1^2}{\log\s{trace}(\calL^2)}.
\end{align}
We next consider two canonical special cases for which we find explicit formulas for the lower bounds in Theorems~\ref{th:LB} and \ref{th:LB2}.

\begin{example}[Gaussian databases]\label{exmp:Gaussian}
    In the Gaussian case, we assume that $Q_{XY} = P_X\times P_Y$, with $P_X$ and $P_Y$ correspond to the densities of a Gaussian random variable with zero mean and unit variance, while $P_{XY}$ is the joint density of two correlated zero mean Gaussian random variables with unit variance, i.e.,
    \begin{align}
        P_{XY}\equiv\mathcal{N}\p{\begin{bmatrix}
0 \\
0 
\end{bmatrix},\begin{bmatrix}
1 & \rho\\
\rho & 1
\end{bmatrix}},
    \end{align}
for some known correlation coefficient $\rho\in[-1,1]\setminus\{0\}$. The Gaussian case above was analyzed in \cite{9834731,nazer2022detecting,tamir,HuleihelElimelech}. A straightforward calculation reveals that,
\begin{align}
    \calL(x,y) = \frac{1}{\sqrt{1-\rho^2}}\exp\p{\frac{-(x^2+y^2)\rho^2+2 x y\rho}{2(1-\rho^2)}}.\label{eqn:GaussianLike}
\end{align}
This is well-known as the Mehler's kernel, and can be diagonalized by the Hermite polynomials as follows \cite{kibble_1945},
\begin{align}
    \calL(x,y) = \sum_{\ell=0}^\infty\frac{\rho^\ell}{\ell!}H_\ell(x)H_\ell(y),
\end{align}
where $\bE_{\s{Z}\sim\calN(0,1)}[H_i(\s{Z})H_{j}(\s{Z})] = i!$, for $i=j$, and zero otherwise. As so, we can deduce that the eigenvalues of $\calL$ in this case are given by $\lambda_\ell = \rho^\ell$, for $\ell\geq0$. Thus, the lower bound in Theorem~\ref{th:LB} for weak detection, boils down to,
 \begin{align}
\sum_{i\geq1}\frac{\rho^{2i}}{1-\rho^{2i}}=o(d^{-1}).\label{eqn:condExample1}     
 \end{align}
To simplify this condition, we note that
\begin{align}
\sum_{i\geq1}\frac{\rho^{2i}}{1-\rho^{2i}}\leq \frac{1}{1-\rho^2}\sum_{i\geq1}\rho^{2i} = \frac{\rho^2}{(1-\rho^2)^2} = o(d^{-1}),\label{eqn:condExample2}     
 \end{align}
 where the last step holds if $\rho^2 = o(d^{-1})$, which of course implies that \eqref{eqn:condExample1} holds under the same condition. This condition, in turn, coincides with \cite{HuleihelElimelech}. Similarly, in the regime where $d$ is fixed, strong detection is impossible when \eqref{eqn:strongdetectlowerCond1} holds, which in the Gaussian setting boils down to
 \begin{align}
     d<\frac{\log \rho^2}{\log(1-\rho^2)},
 \end{align}
 which again coincides with the results of \cite{HuleihelElimelech}. Therefore, our general lower bounds in Theorems~\ref{th:LB} and \ref{th:LB2} recover the known bounds in the literature.
 \end{example}

\begin{example}[Bernoulli databases]\label{exmp:Bernoulli} In the Bernoulli case, we assume that $Q_{XY} = P_X\times P_Y$, with 
$P_X=P_Y = \s{Bernoulli}(\tau p)$, for some $p\in(0,1)$ and $\tau\in[0,1]$, and $P_{XY}$ denotes the joint distribution of two correlated Bernoulli random variables. Specifically, under $P_{XY}$, we have $X\sim\s{Bernoulli}(\tau p)$, and
\begin{align}
    Y\vert X&\sim\begin{cases}
    \s{Bernoulli}(\tau),\ &\s{if }\;X=1\\ 
    \s{Bernoulli}\p{\frac{\tau p(1-\tau)}{1-\tau p}},\ &\s{if }\;X=0.
    \end{cases}
\end{align}
Here, Pearson correlation coefficient is given by,
\begin{align}
    \rho\triangleq\frac{\s{cov}(X,Y)}{\sqrt{\s{var}(X)}\sqrt{\s{var}(Y)}} = \frac{\tau(1-p)}{1-\tau p}.
\end{align}
The Bernoulli case was not studied in the literature in the context of the database alignment detection problem, but it is the focus in related work on testing correlation of unlabeled random graphs, e.g., \cite{wu2020testing}, where the edges are modelled as Bernoulli random variables. Again, a straightforward calculation reveals that in the above setting,
\begin{align}
    \calL(x,y) = \begin{cases}
        \frac{1}{p},\ &\s{if}\; (x,y) = (1,1)\\
        \frac{1-\tau}{1-\tau p},\ &\s{if}\; (x,y) = (0,0)\\
        \frac{1-2\tau p+\tau^2p}{(1-\tau p)^2},\ &\s{otherwise}.
    \end{cases}
\end{align}
Then, the eigenvalues of $\calL$  correspond to the eigenvalues of the $2\times 2$ matrix $\calM(x,y) \triangleq P_{XY}(x,y)/Q_X(x)$, for $x,y\in\{0,1\}$. Explicitly, we have
\begin{align}
    \calM = \begin{bmatrix}
\frac{1-p\tau(2-\tau)}{1-p\tau} & \frac{p\tau(1-\tau)}{1-p\tau}\\
1-\tau & \tau
\end{bmatrix}.
\end{align}
Accordingly, it can be shown that the eigenvalues of $\calM$ are $1$ and $\rho$. Thus, the lower bound in Theorem~\ref{th:LB} for weak detection, boils down to,
 \begin{align}
\frac{\rho^2}{1-\rho^2}=o(d^{-1}).\label{eqn:condExample3}     
 \end{align}
which is equivalent to $\rho^2 = o(d^{-1})$, as in the Gaussian case. Similarly, in the regime where $d$ is fixed, strong detection is impossible when \eqref{eqn:strongdetectlowerCond1} holds, which in the Bernoulli setting, translates to,
 \begin{align}
     d<\frac{\log(1/\rho^2)}{\log(1+\rho^2)},
 \end{align}
 similarly to the Gaussian case.
 \end{example}

Next, we present our detection algorithms and the corresponding upper bounds. We start with the observation that the composite hypothesis testing problem in \eqref{eqn:compositeTesting} can be reduced/formulated as a simple hypothesis testing problem where the latent permutation $\sigma$ is drawn uniformly at random over $\mathbb{S}_n$. Accordingly, as is well-known the optimal decision rule is the Neyman-Pearson test which correspond to thresholding the likelihood ratio,
\begin{align}
    \frac{\pr_{\calH_1}(\s{X},\s{Y})}{\pr_{\calH_0}(\s{X},\s{Y})}=\bE_{\sigma\sim\s{Unif}(\mathbb{S}_n)}\pp{\frac{\pr_{\calH_1\vert\sigma}(\s{X},\s{Y}\vert\sigma)}{\pr_{\calH_0}(\s{X},\s{Y})}}.
\end{align}
As is customary to other related testing problems with a latent combinatorial structure, analyzing the above optimal test is difficult due to the averaging over all $n!$ possible permutations. Consider, instead, the generalized likelihood ratio test (GLRT), where the averaging is replaced with a maximum, namely, 
\begin{align}
    \max_{\sigma\in\mathbb{S}_n}\frac{\pr_{\calH_1\vert\sigma}(\s{X},\s{Y}\vert\sigma)}{\pr_{\calH_0}(\s{X},\s{Y})} = \max_{\sigma\in\mathbb{S}_n}\prod_{i=1}^n\frac{P^{\otimes d}_{XY}(X_{i},Y_{\sigma_i})}{Q^{\otimes d}_{XY}(X_{i},Y_{i})},
\end{align}
and accordingly, we define the test as,
\begin{align}
    \phi_{\s{GLRT}}(\s{X},\s{Y})&\triangleq\Ind\ppp{\max_{\sigma\in\mathbb{S}_n}\frac{1}{dn}\sum_{i=1}^n\log\frac{P^{\otimes d}_{XY}(X_{i},Y_{\sigma_i})}{Q^{\otimes d}_{XY}(X_{i},Y_{i})}\geq\tau_{\s{GLRT}}},\label{eqn:testscanmain}
\end{align}
for some threshold $\tau_{\s{GLRT}}\in\mathbb{R}$. To present our results, we need a few definitions and notations. Let us denote the log-likelihood ratio (LLR) by $\mathscr{L}(x,y)\triangleq\log\calL(x,y) = \log\frac{P_{XY}(x,y)}{Q_{XY}(x,y)}$, for $(x,y)\in\calX\times\calY$, i.e., the logarithm of the Radon-Nikodym derivative between $P_{XY}$ and $Q_{XY}$. We
assume they are mutually absolutely continuous for simplicity and so $\mathscr{L}$ is well-defined. We also
assume that the expectations of the LLR with respect to $P_{XY}$ and $Q_{XY}$ are finite, or, in other words, that $d_{\s{KL}}(P_{XY}||Q_{XY})$ and $d_{\s{KL}}(Q_{XY}||P_{XY})$ are finite. For $\theta\in(-d_{\s{KL}}(Q_{XY}||P_{XY}),d_{\s{KL}}(P_{XY}||Q_{XY}))$, we define the Chernoff's exponents $E_{P},E_Q:\mathbb{R}\to[-\infty,\infty)$ as the Legendre transforms of the log-moment generating functions, namely,
\begin{align}
    E_Q(\theta)\triangleq\sup_{\lambda\in\mathbb{R}}\lambda\theta-\psi_Q(\lambda),\quad,E_P(\theta)\triangleq\sup_{\lambda\in\mathbb{R}}\lambda\theta-\psi_P(\lambda),
\end{align}
where $\psi_Q(\lambda)\triangleq\log\bE_Q[\exp(\lambda\mathscr{L})]$ and $\psi_P(\lambda)\triangleq\log\bE_P[\exp(\lambda\mathscr{L})]$. Note that $\psi_P(\lambda) = \psi_Q(\lambda+1)$, and thus $E_P(\theta) = E_Q(\theta)-\theta$. In particular, $E_P$ and $E_Q$ are non-negative convex functions. Moreover, since $\psi'_Q(0) = -d_{\s{KL}}(Q_{XY}||P_{XY})$ and $\psi'_Q(1) = d_{\s{KL}}(P_{XY}||Q_{XY})$, we have $E_Q(-d_{\s{KL}}(Q_{XY}||P_{XY})) = E_P(d_{\s{KL}}(P_{XY}||Q_{XY}))=0$ and hence $E_Q(d_{\s{KL}}(P_{XY}||Q_{XY})) = d_{\s{KL}}(P_{XY}||Q_{XY})$ and $E_P(-d_{\s{KL}}(Q_{XY}||P_{XY}))=d_{\s{KL}}(Q_{XY}||P_{XY})$. Finally, it is well known that $\lambda\theta-\psi_Q(\lambda)$ is concave
and has derivative at zero given by $\theta+d_{\s{KL}}(Q_{XY}||P_{XY})$. This implies that the maximizer $\lambda^\star$ to the concave optimization $E_Q(\theta)$ can be taken to be non-negative if $\theta\geq - d_{\s{KL}}(Q_{XY}||P_{XY})$. The same is true for $E_P(\theta)$ if  $\theta\leq d_{\s{KL}}(P_{XY}||Q_{XY})$.

\begin{theorem}[GLRT strong detection]\label{thm:2}
Consider the GLRT in \eqref{eqn:testscanmain}, and suppose there is a $\tau_{\s{GLRT}}\in(-d_{\s{KL}}(Q_{XY}||P_{XY}),d_{\s{KL}}(P_{XY}||Q_{XY}))$ with
\begin{subequations}\label{eqn:GLRTcond}
	\begin{align}
		E_Q(\tau_{\s{GLRT}})&\geq \frac{\log (n/e)}{d}+\frac{1+\log n}{dn}+\omega(d^{-1}n^{-1}) ,\label{eqn:GLRTcond1} \\
		E_P(\tau_{\s{GLRT}})&= \omega(d^{-1}n^{-1}) .\label{eqn:GLRTcond2}
	\end{align}
\end{subequations}
Then, $\s{R}(\phi_{\s{GLRT}})\to0$, as $n,d\to\infty$, and $d=\omega(\log n)$.
\end{theorem}
As an alternative test for the GLRT, we propose also the following simple sum test,
\begin{align}
\phi_{\s{sum}}(\s{X},\s{Y})\triangleq\Ind\ppp{\frac{1}{dn^2}\sum_{i,j=1}^n\sum_{\ell=1}^d\calK(X_{i\ell},Y_{j\ell})\geq\tau_{\s{sum}}},\label{eqn:testSummain}
\end{align}
where
\begin{align}
    \calK(X_{i\ell},Y_{j\ell})&\triangleq\log\frac{P_{XY}(X_{i\ell},Y_{j\ell})}{Q_{XY}(X_{i\ell},Y_{j\ell})}-\bE_{A\sim P_X}\pp{\log\frac{P_{XY}(A,Y_{i\ell})}{Q_{XY}(A,Y_{i\ell})}}-\bE_{B\sim P_X}\pp{\log\frac{P_{XY}(X_{i\ell},B)}{Q_{XY}(X_{i\ell},B)}}\nonumber\\
    &\quad-d_{\s{KL}}(Q_{XY}||P_{XY}),\label{eqn:calkDef}
\end{align}
is the centered likelihood function, and $\tau_{\s{sum}}\in\mathbb{R}$. In the Gaussian case (as well as in the Bernoulli case) it can be shown that \eqref{eqn:testSummain} is equivalent to thresholding the sum of all entries of the inner product $\s{X}^T\s{Y}$, i.e., $\sum_{i,j=1}^nX_{ij}Y_{ij}$. This test was analyzed in \cite{nazer2022detecting}, in the Gaussian case. For the general case, we have the following result. 
\begin{theorem}[Sum test]\label{thm:3}
Consider the sum test in \eqref{eqn:testSummain}, and let 
\begin{align}
    \tau_{\s{sum}}=dn\cdot d_{\s{SKL}}(P_{XY}||Q_{XY}).
\end{align}
Then,
\begin{align}
\s{R}(\phi_{\s{sum}})\leq \frac{4\cdot \s{Var}_{Q_{XY}}\p{\calK(A,B)}}{d\cdot d^2_{\s{SKL}}(P_{XY}||Q_{XY})}.
\end{align}
In particular, if
\begin{align}
    d\cdot\frac{d^2_{\s{SKL}}(P_{XY}||Q_{XY})}{\s{Var}_{Q_{XY}}\p{\calK(A,B)}}=\omega(1),\label{eqn:sumtestcond}
\end{align}
then, $\s{R}(\phi_{\s{sum}})\to0$, as $d\to\infty$.
\end{theorem}
Note that Theorem~\ref{thm:3} implies also that weak detection is possible using the sum test if 
\begin{align}
    d\cdot\frac{d^2_{\s{SKL}}(P_{XY}||Q_{XY})}{\s{Var}_{Q_{XY}}\p{\calK(A,B)}} = \Omega(1).
\end{align}
At this point, it is evident that based on the bounds in Theorems~\ref{thm:2} and \ref{thm:3}, the GLRT and the sum test can achieve strong detection (vanishing risk) only if $d\to\infty$, namely, if $d=O(1)$, then the bounds on the risks associated with the GLRT and sum test are not vanishing as $n\to\infty$. Accordingly, pursuing for strong detection in the scenario where $d$ is constant, we propose the following testing procedure,
\begin{align}
    \phi_{\s{count}}(\s{X},\s{Y})\triangleq\Ind\ppp{\sum_{i,j=1}^n\Ind\ppp{\frac{1}{d}\log\frac{P^{\otimes d}_{XY}(X_{i},Y_{j})}{Q^{\otimes d}_{XY}(X_{i},Y_{j})}\geq\tau_{\s{count}}}\geq \frac{1}{2} n\calP_{d}},\label{eqn:testcount}
\end{align}
where $\calP_{d}\triangleq\pr_{P_{XY}^{\otimes d}}\pp{\sum_{\ell=1}^d\mathscr{L}(A_\ell,B_\ell)\geq d\cdot\tau_{\s{count}}}$, and $\tau_{\s{count}}\in\mathbb{R}$. Roughly speaking, $\phi_{\s{count}}$ counts the number of pairs whose likelihood individually exceed a certain threshold. This is similar (but not exactly the same) to a test proposed in \cite{HuleihelElimelech} in the Gaussian setting, which counts the number of inner products between all possible (normalized) rows in $\s{X}$ and $\s{Y}$ who individually exceed a certain threshold. We mention here that our result holds for any natural $d\geq1$, while in \cite{HuleihelElimelech} it is assumed that $d\geq d_0$, for some fixed $d_0\in\mathbb{N}$ (most notably, excluding the interesting $d=1$ case). We have the following result.
\begin{theorem}[Count test strong detection]\label{thm:IndSum}
Fix $d\in\mathbb{N}$, and consider the count test in \eqref{eqn:testcount}. Suppose there is a $\tau_{\s{count}}\in(-d_{\s{KL}}(Q_{XY}||P_{XY}),d_{\s{KL}}(P_{XY}||Q_{XY}))$ with
\begin{subequations}\label{eqn:Countcond}
	\begin{align}
    E_Q(\tau_{\s{count}}) &= \omega\p{\log n^{1/d}},\label{eqn:Countcond1}\\
    E_P(\tau_{\s{count}}) &= \omega(n^{-1}d^{-1}).\label{eqn:Countcond2}
\end{align}
\end{subequations}
Then, $\s{R}(\phi_{\s{count}})\to0$, as $n\to\infty$.
\end{theorem}

As for the lower bounds, we provide examples for which we derive explicit formulas for the upper bounds above. 

\begin{example}[Gaussian databases] Consider the same setting as in Example~\ref{exmp:Gaussian}, and let us find the corresponding Chernoff's exponents in this case. Recall that the likelihood function is given in \eqref{eqn:GaussianLike}. Then,
\begin{align}
    \psi_Q(\lambda) &= \log\bE_Q\pp{\frac{1}{(1-\rho^2)^{\lambda/2}}\exp\p{\lambda\frac{-(x^2+y^2)\rho^2+2 x y\rho}{2(1-\rho^2)}}}\\
    & = -\frac{\lambda}{2}\log(1-\rho^2)+\log\bE_Q\pp{\exp\p{\lambda\frac{-(x^2+y^2)\rho^2+2 x y\rho}{2(1-\rho^2)}}}.
\end{align}
Using the fact that the moment generating function of Gaussian random variable $\s{W}\sim\calN(\mu,\sigma^2)$ squared is given by,
\begin{align}
    \bE [\exp\p{t\s{W}^2}] = \frac{1}{\sqrt{1-2t\sigma^2}} \exp\left(\frac{\mu^2 t}{1-2t\sigma^2}\right),
\end{align}
for $\s{Real}(t\sigma^2)<1/2$, it can be shown that
\begin{align}
    \bE_Q\pp{\exp\p{\lambda\frac{-(x^2+y^2)\rho^2+2 x y\rho}{2(1-\rho^2)}}} = \frac{\sqrt{1-\rho^2}}{\sqrt{1-(1-2\lambda)^2\rho^2}},
\end{align}
for $-\frac{1-\rho^2}{\rho^2}\leq\lambda$ and $\lambda[1-(1-(1-\lambda)\rho^2)^{-1}]<1$. Thus,
\begin{align}
  \psi_Q(\lambda) = -\frac{\lambda-1}{2}\log(1-\rho^2)-\frac{1}{2}\log[1-(1-2\lambda^2)\rho^2].  
\end{align}
Let us take $\tau_{\s{GLRT}}=0$. Then,
\begin{align}
    E_Q(0) &= \sup_{\lambda\in\mathbb{R}}-\psi_Q(\lambda)\geq-\psi_Q(1/2) = -\frac{1}{4}\log(1-\rho^2).\label{eqn:EQExamp}
\end{align}
Therefore, we see that for \eqref{eqn:GLRTcond1} to hold it is suffice that $d\rho^2=\omega(\log n)$. Similarly, it can be shown that
\begin{align}
    \psi_P(\lambda) = -\frac{\lambda}{2}\log(1-\rho^2)-\frac{1}{2}\log(1-\lambda^2\rho^2),
\end{align}
for $-1/\rho^2<\lambda<1/\rho^2$. Then,
\begin{align}
    E_P(0) &= \sup_{\lambda\in\mathbb{R}}-\psi_P(\lambda)\geq-\psi_P(-1/2) = -\frac{1}{4}\log(1-\rho^2)+\frac{1}{2}\log(1-\rho^2/4).
\end{align}
Accordingly, we see that for \eqref{eqn:GLRTcond2} to hold it is suffice that $d\rho^2=\omega(n^{-1})$. Hence, in the Gaussian special case, for the GLRT to achieve strong detection it is suffice that $d\rho^2=\omega(\log n)$. In light of Example~\ref{exmp:Gaussian}, this is tight up to a $\log n$ factor. Next, for the sum test in Theorem~\ref{thm:3}, a straightforward calculation shows that
\begin{align}
    d_{\s{KL}}(P_{XY}||Q_{XY}) &= -\frac{1}{2}\log(1-\rho^2),\\
    d_{\s{KL}}(Q_{XY}||P_{XY}) &= \frac{1}{2}\log(1-\rho^2)+\frac{\rho^2}{1-\rho^2}.
\end{align}
Furthermore,
\begin{align}
    \calK(A,B) = \frac{\rho}{1-\rho^2}\cdot AB,
\end{align}
and thus $\s{Var}_{Q_{XY}}\p{\calK(A,B)} = \frac{\rho^2}{(1-\rho^2)^2}$. Therefore, the condition in \eqref{eqn:sumtestcond} boils down to $d\rho^2=\omega(1)$, which in light of Example~\ref{exmp:Gaussian} is tight up to a constant term. Finally, for fixed $d\in\mathbb{N}$, we see from \eqref{eqn:Countcond} and \eqref{eqn:EQExamp} that a sufficient condition for the count test to achieve strong detection is $\rho^2 = 1-o(n^{-4/d})$, as $n\to\infty$. Interestingly, this bound coincided with the threshold for the recovery problem \cite{pmlr-v89-dai19b}, achieved by the exhaustive maximum-likelihood estimator, while the count test is clearly efficient.
\end{example}
\begin{example}[Bernoulli databases]
We can repeat the same calculations for the Bernoulli case. As an example, consider the sum test. Here, a straightforward calculation reveals that,
\begin{align}
        d_{\s{KL}}(P_{XY}||Q_{XY}) &= O(\rho^2),\\
    d_{\s{KL}}(Q_{XY}||P_{XY}) &= O(\rho^2),
\end{align}
and $\s{Var}_{Q_{XY}}\p{\calK(A,B)} = O(\rho^2)$. Therefore, the condition in \eqref{eqn:sumtestcond} boils down to $d\rho^2=\omega(1)$, which in light of Example~\ref{exmp:Bernoulli} is tight up to a constant term. Similarly, when $d$ is fixed it can be shown that the count test achieves strong detection if the same condition as in the Gaussian case holds.
\end{example}

We conclude this section by showing some numerical evaluations in the Gaussian setting (see, Example~\ref{exmp:Gaussian}). Specifically, in Fig.~\ref{fig:risk}, we present the empirical risk, averaged over $2\times 10^3$ Monte-Carlo runs, associated with the sum and count tests, as a function of $\rho$, for different values of $d$ in Fig.~\ref{fig:sub-first} and $n$  in Fig.~\ref{fig:sub-second}. As predicted by our theoretical results, it can be seen that for large values of $d$, the risk associated with the sum test is vanishing if the correlation is at least some value, independent of $n$, while if $d$ is relatively small, then the risk is not vanishing for any value of $\rho$ and $n$. On the other hand, this is not the case for the count test, where even for small values of $d$, if the correlation is sufficiently close to unity, then the associated risk decreases. 

\begin{figure}[t]
\begin{subfigure}{.48\textwidth}
  \centering
  % include first image
  \includegraphics[width=1\linewidth]{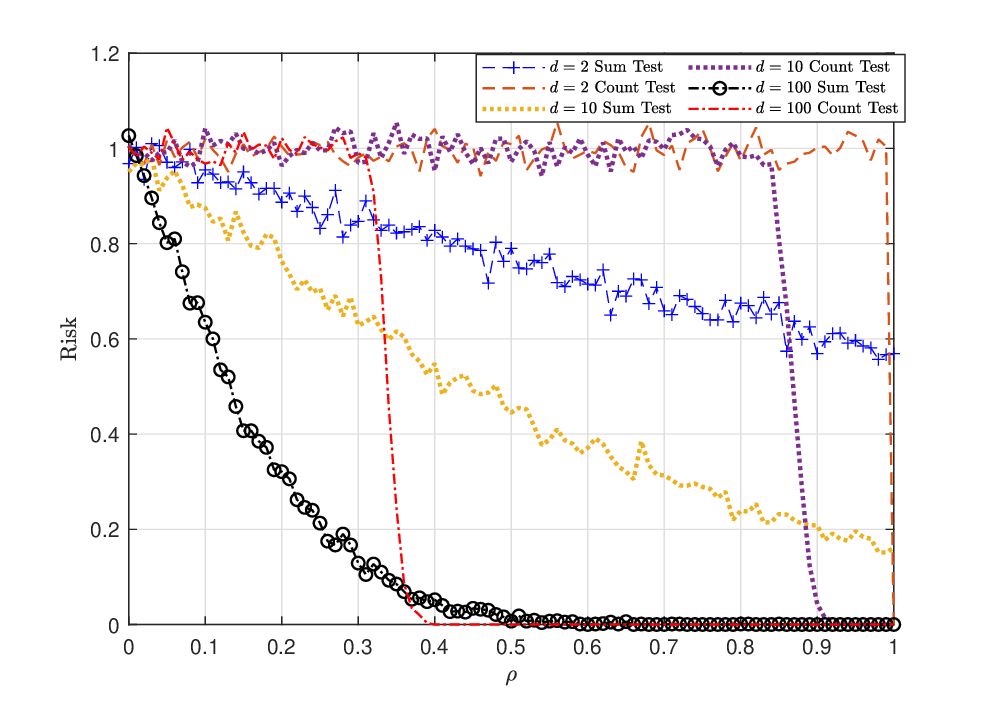}  
  \caption{The empirical risk associated with the sum and count tests as a function of $\rho$, for $n=100$, and various values of $d\in\{2,10,100\}$.}
  \label{fig:sub-first}
\end{subfigure}\hfill
\begin{subfigure}{.48\textwidth}
  \centering
  % include second image
  \includegraphics[width=1\linewidth]{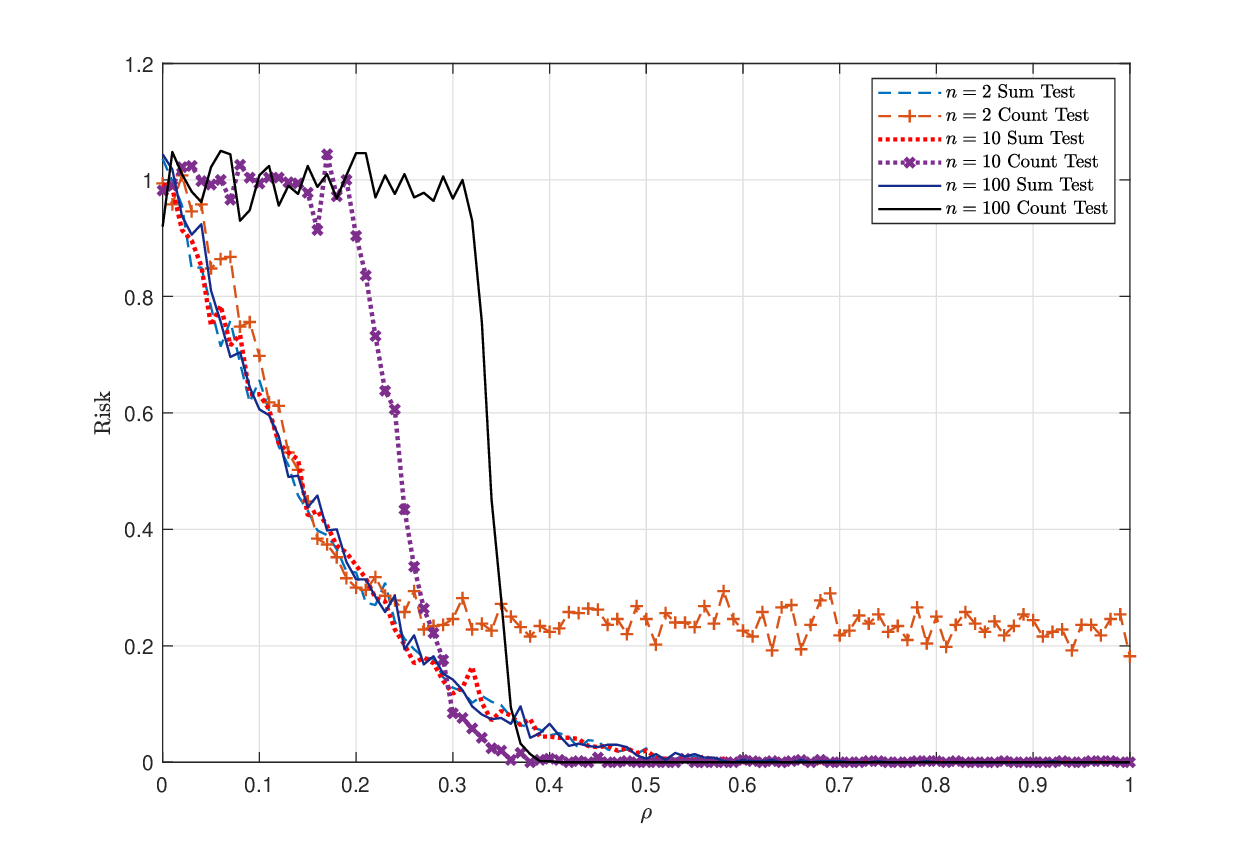}  
  \caption{The empirical risk associated with the sum and count tests as a function of $\rho$, for $d=100$, and various values of $n\in\{2,10,100\}$.}
  \label{fig:sub-second}
\end{subfigure}
\caption{Empirical risks as a function of $\rho$.}
\label{fig:risk}
\end{figure}

\section{Proofs of Lower Bounds}\label{sec:proofLowerBounds}

As is customary in lower bounds proofs for composite (worst-case) detection problems, it is convenient to reduce to a simple hypothesis testing problem by inducing a uniform prior on the parameter space. %    in many detection problems, evaluating the minimax risk function  opposes a great challenge due to the error term obtained by maximizing over the error for all permutations in $\S_n$. A well known strategy for overcoming this inherent obstacle is by considering the softer average-case version of the problem. 
Specifically, let $\pi$ be the uniform measure on $\S_n$, and let us denote by $\P_{\calH_1}$ the probability measure obtained by averaging $\P_{\calH_1\vert\pi}$ with respect to $\pi$. For a sequence of tests $\phi_{n,d}$, we define the average/Bayesian risk by,
\begin{align}
\bar{\s{R}}(\phi_{n,d}) \triangleq \P_{\calH_0}[\phi_{n,d}(\s{X},\s{Y})=1]+\E_{\sigma\sim\pi}\pp{\P_{\calH_1\vert\sigma}[\phi_{n,d}(\s{X},\s{Y})=0]},
\end{align}
and the optimal average risk,
\begin{align}
\bar{\s{R}}^\star\triangleq \inf_{\phi_{n,d}:\calX^{n\times d}\times\calY^{n\times d}\to \{0,1\}}\bar{\s{R}}(\phi_{n,d}).
\end{align}

Clearly, any test $\phi_{n,d}$ satisfies $\s{R}(\phi_{n,d})\geq \bar{\s{R}}(\phi_{n,d})$, and therefore $\s{R}^\star\geq \bar{\s{R}}^\star$. We conclude that lower bounds on the average risk imply lower bounds on the worst-case risk. For example, in order to prove Theorem~\ref{th:LB}, it is sufficient to find conditions for which $\bar{\s{R}}^\star=1-o(1)$. In particular, using a well-known equivalent characterization of the Bayesian risk function by the total variation distance and Cauchy-Schwartz inequality, we have,
\begin{equation}
    \label{eq:LikeBound2}
    {\s{R}}^\star\geq \bar{\s{R}}^\star= 1-d_{\s{TV}}(\P_{\calH_0},\P_{\calH_1})\geq 1-\frac{1}{2}\sqrt{\E_{\calH_0}\pp{\s{L}_n^2}-1},
\end{equation}
where $\s{L}_n\triangleq\frac{\P_{\calH_1}}{\P_{\calH_0}}$ is the likelihood ratio, and the expectation is taken with respect to $\P_{\calH_0}$. Accordingly, using \eqref{eq:LikeBound2} we see that in order to prove that $(\calD,d,n)$ is not admissible for weak detection, it is sufficient to show that under the assumptions of Theorem~\ref{th:LB}, $\E_{\calH_0}\pp{\s{L}_n^2}\leq 1+o(1)$. For strong detection, on the other hand, recall that (see, for example, \cite[Lemmata 2.6 and 2.7]{tsybakov2004introduction}), 
\begin{align}
    \bE_{\calH_0}[\s{L}_n^2] = O(1)\implies d_{\s{TV}}(\P_{\calH_0},\P_{\calH_1})=1-\Omega(1),
\end{align}
which implies, 
\begin{equation}\label{eq:ImpStrongDet0}
    \s{R}^\star\geq 1-d_{\s{TV}}(\P_{\calH_0},\P_{\calH_1})=\Omega(1),
\end{equation}
and thus $(\calD,d,n)$ is not admissible for strong detection (as \eqref{eq:ImpStrongDet0} implies $\limsup\s{R}^\star>0$). Therefore, our main goal now is to analyze $\E_{\calH_0}\pp{\s{L}_n^2}$.

\subsection{A formula for the second moment}

The first step in the calculation calls for a use of Ingster-Suslina method, stating that by Fubini's theorem, $\E_{\calH_0}[\s{L}_n^2]$ may be equivalently written as
\begin{equation}
    \label{eq:LikeSquare2}\E_{\calH_0}[\s{L}_n^2]=\E_{\pi\indep\pi'}\pp{\E_{\calH_0}\pp{\frac{\P_{\calH_1|\pi}}{\P_{\calH_0}}\cdot \frac{\P_{\calH_1|\pi'}}{\P_{\calH_0}}}},
\end{equation}
where the expectation is taken with respect to the independent coupling of $\pi$ and $\pi'$, two copies of the uniform measure on $\S_n$. Note that for fixed permutations $\pi$ and $\pi'$, in the discrete case, the distributions $\P_{\calH_1\vert\pi},\P_{\calH_1|\pi'}$ and $\P_{0}$, correspond to the underlying probability mass functions under the alternative and null hypotheses. In the continuous case, $\P_{\calH_1\vert\pi},\P_{\calH_1|\pi'}$ and $\P_{0}$ are absolutely continuous with respect to Lebesgue's measure on $\R^{2\times d\times n}$, and therefore,
\begin{align}
\frac{\P_{\calH_1\vert\pi}}{\P_{\calH_0}}=\frac{f_{\calH_1|\pi}}{f_{\calH_0}}, \quad\text{and} \quad \frac{\P_{\calH_1|\pi'}}{\P_{\calH_0}}=\frac{f_{\calH_1|\pi'}}{f_{\calH_0}}, 
\end{align}
where $f_{\calH_i}$ denotes the Radon-Nikodym derivative of $\P_{\calH_i}$ with respect to Lebesgue's measure, which is the density function $(\s{X},\s{Y})$ under the corresponding hypothesis. In the following, with some abuse of notations, we treat and use the same notations for both cases. Specifically, given $\pi,\pi'$, we note that
\begin{align}
    \frac{\P_{\calH_1|\pi}}{\P_{\calH_0}}\cdot \frac{\P_{\calH_1|\pi'}}{\P_{\calH_0}} &= \prod_{i=1}^n \s{L}(X_i,Y_{\pi_i})\s{L}(X_i,Y_{\pi_i'})\triangleq \prod_{i=1}^n\s{Z}_{i},\label{eqn:ZiDef}
\end{align}
where
\begin{align}
\s{L}(X_i,Y_{\pi_i})\triangleq\frac{P_{XY}^{\otimes d}(X_i,Y_{\pi_i})}{Q_{XY}^{\otimes d}(X_i,Y_{\pi_i})},  
\end{align}
for any $i\in[n]$.
% $Q_{XY}$ denotes the joint probability mass/density function of $X_i$ and $Y_i$ under $\calH_0$, and $P_{XY}$ denotes the joint probability mass/density function of $X_i$ and $Y_{\pi_i}$ under $\calH_1$ given its latent permutation $\pi$. 
Thus, substituting \eqref{eq:LikeSquare2} in \eqref{eqn:ZiDef}, we get,
\begin{align}
    \E_{\calH_0}[\s{L}_n^2]=\E_{\pi\indep\pi'}\pp{\E_{\calH_0}\p{\prod_{i=1}^n\s{Z}_{i}}}.\label{eqn:Product}
\end{align}
Next, we find an expression for the inner expectation in \eqref{eqn:Product}, for fixed permutations $\pi$ and $\pi'$. Observe that $\s{Z}_{i}$'s may not be independent across different values of $i$. Nonetheless, using the notion of cyclic decomposition, we can decompose $\prod_{i}\s{Z}_{i}$ as a product over independent sets of random variables. To that end, define $\sigma \triangleq\pi^{-1}\circ\pi'$, and note that $\sigma$ is distributed uniformly over $\mathbb{S}_n$. Now, for each orbit $O$ of $\sigma$, define
\begin{align}
    \s{Z}_O\triangleq\prod_{i\in O}\s{Z}_{i}.\label{eqn:ZoDef}
\end{align}
The main observation here is that $\s{Z}_O$ is a function of $(X_i,Y_{\pi_i})_{i\in O}$ solely. Let $\calO$ denote the collection of all possible orbits of $\sigma$. By definition, orbits are disjoint, and therefore we may write,
\begin{align}
   \prod_{i=1}^n\s{Z}_{i} = \prod_{O\in\calO}\s{Z}_{O}.
\end{align}
Since the random vectors $\{X_i\}_{i\in[n]}$ and $\{Y_i\}_{i\in[n]}$ are statistically independent under $\calH_0$, the random variables $\{\s{Z}_O\}_{O\in\calO}$ are also mutually independent under $\calH_0$. Therefore, we obtain that the product in \eqref{eqn:Product} can be rewritten as the following product over independent sets of random variables,
\begin{align}
    \E_{\calH_0}[\s{L}_n^2]=\E_{\pi\indep\pi'}\pp{ \prod_{O\in\calO}\E_{\calH_0}[\s{Z}_{O}]}.\label{eqn:secORb}
\end{align}
\begin{rmk}
    The fact that the inner expectation in \eqref{eqn:Product} depend $\pi$ and $\pi'$ only through $\sigma$ is clear. Indeed,  note that the distribution under the null hypothesis is invariant to the reordering of the coordinates, and the uniform measure on $\mathbb{S}_n$, is invariant under composition with a fixed permutation. 
\end{rmk}

Next, we find an expression for $\E_{\calH_0}[\s{Z}_{O}]$. To that end, recall our definitions and notations right before the statement of Theorem~\ref{th:LB}. Notice to the difference between $\calL(\cdot,\cdot)$ and $\s{L}(\cdot,\cdot)$, the former being the $d=1$ special case of the latter. %This kernel defines an operator as follows: for any square-integrable function $f$ under $Q_{XY}$,
%\begin{align}
 %   (\calL f)(x) \triangleq \bE_{Y\sim Q_Y}\pp{\calL(x,Y)f(Y)},
%\end{align}
%In addition, $\calL^2= \calL\circ\calL$ is given by $\calL^2(x,y) = \bE_{Z\sim Q}[\calL(x,Z)\calL(Z,y)]$, and $\calL^k$ is similarly defined. Assume that $\calL(x,y) = \calL(y,x)$, and hence $\calL$ is self-adjoint. Furthermore, if we assume that $ \int\int \calL^2(x,y)Q_{X}(\mathrm{d}x)Q_{Y}(\mathrm{d}y)<\infty$, then $\calL$ is Hilbert-Schmidt. Thus $\calL$ is diagonazable
%with eigenvalues $\lambda_i$'s and the trace of $\calL$ is given by $\s{trace}(\calL) = \bE_{Y\sim Q_{Y}}[\calL(Y,Y)] = \sum_{i\in\mathbb{N}}\lambda_i$. 
The following result gives a formula for the inner expectation in the right-hand-side of \eqref{eqn:secORb}, as a function of the eigenvalues of $\calL$. The proof of the following result relies on similar arguments as in \cite{wu2020testing}, for the special cases of the Bernoulli and Gaussian distributions over random graphs.
\begin{lemma}\label{lem:secoLOrbit}
    Fixing $\pi$ and $\pi'$, for any orbit O of $\sigma \triangleq\pi^{-1}\circ\pi'$, we have
    \begin{align}
        \E_{\calH_0}[\s{Z}_{O}] = \p{\sum_{i\in\mathbb{N}}\lambda_i^{2|O|}}^d.
    \end{align}
\end{lemma}
\begin{proof}[Proof of Lemma~\ref{lem:secoLOrbit}]
Let $k = |O|$. Since $O$ is an orbit of $\sigma$, we have $\{Y_{\pi_i}\}_{i\in O} = \{Y_{\pi'_i}\}_{i\in O}$ and $\pi'_i = \pi_{\sigma_i}$. Therefore, using \eqref{eqn:ZiDef} and \eqref{eqn:ZoDef},
\begin{align}
   \E_{\calH_0}[\s{Z}_{O}] &= \E_{\calH_0}\pp{\prod_{i\in O} \s{L}(X_i,Y_{\pi_i})\s{L}(X_i,Y_{\pi_i'})}\\
   & =\p{\E_{\calH_0}\pp{\prod_{i\in O} \calL(X_i,Y_{\pi_i})\calL(X_i,Y_{\pi_i'})}}^d\\
   & = \p{\bE\pp{\prod_{i=1}^{|O|}\calL(X_i,Y_i)\calL(X_i,Y_{(i+1)\s{mod}|O|})}}^d\\
   & = \p{\bE\pp{\prod_{i=1}^{|O|}\calL^2(Y_i,Y_{(i+1)\s{mod}|O|})}}^d\\
   & = \p{\s{trace}\p{\calL^{2|O|}}}^d\\
   & = \p{\sum_{i\in\mathbb{N}}\lambda_i^{2|O|}}^d.
\end{align}
\end{proof}
Recall that $\sigma \triangleq\pi^{-1}\circ\pi'$ and as so $\sigma\sim\s{Unif}(\mathbb{S}_n)$. For a fixed permutation $\pi\in \S_n$ and $k\in [n]$, let $N_k(\pi)$ denote the number of $k$-cycles of $\pi$. Combining \eqref{eqn:secORb}, and Lemma~\ref{lem:secoLOrbit} we obtain
\begin{equation}
    \label{eq:PermProd2}
\E_{\calH_0}[\s{L}_n^2]=\E_{\pi\indep\pi'}\pp{\prod_{O\in\calO}\E_{\calH_0}[\s{Z}_O]}=\E_{\sigma}\pp{\prod_{k=1}^n\p{\sum_{i\in\mathbb{N}}\lambda_i^{2k}}^{dN_k(\sigma)}}.
\end{equation}
For simplicity of notation, in the following we remove the dependency of $N_k(\sigma)$ on $\sigma$, and define $g_k(Q_{XY},P_{XY})\triangleq\sum_{i\in\mathbb{N}}\lambda_i^{2k}$. Then, we have,
\begin{align}
    \E_{\calH_0}[\s{L}_n^2] &= \E_{\sigma}\pp{\prod_{k=1}^ne^{dN_k\log g_k(Q_{XY},P_{XY})}} \\
    &= \E_{\sigma}\pp{\exp\p{d\sum_{k=1}^nN_k\log g_k(Q_{XY},P_{XY})}}.\label{eqn:seocndMoment}
\end{align}
Before we continue let us prove the following result.
\begin{lemma}\label{lem:eigen}
    The largest eigenvalue of $\calL$ is 1.
\end{lemma}
\begin{proof}[Proof of Lemma~\ref{lem:eigen}]
We first prove the result for the case where $P_{XY}$ and $Q_{XY}$ are probability mass functions, i.e., the datasets $\s{X}$ and $\s{Y}$ are discrete and take values over $\calX$ and $\calY$, respectively. The proof for the continuous case, where $P_{XY}$ and $Q_{XY}$ are probability density functions, is similar. Note that for any square-integrable function $f$ under $Q_{Y}$,
\begin{align}
    (\calL f)(x) &= \bE_{Y\sim Q_{Y}}\pp{\calL(x,Y)f(Y)}\\
    & = \sum_{y\in\calY}\frac{P_{XY}(x,y)}{Q_{X}(x)Q_{Y}(y)}f(y)Q(y)\\
    & = \frac{1}{Q_{X}(x)}\sum_{y\in\calY}f(y)P_{XY}(x,y).
\end{align}
Thus, the eigenvalues of $\calL$ are given by the eigenvalues of the following $|\calX|\times|\calY|$ row-stochastic matrix
$\mathbf{M}(x,y)\triangleq \frac{P_{XY}(x,y)}{Q_{X}(x)}$. Since $\mathbf{M}$ is a row-stochastic matrix its largest eigenvalue is $1$. Indeed, let $\mathbf{1}$ be the all one vector. Then, note that $\mathbf{M}\mathbf{1} = \mathbf{1}$, and thus $1$ is an eigenvalue of $\mathbf{M}$. To show that the this is the largest eigenvalue we use Gershgorin circle theorem. Specifically, take row $\ell$ in $\mathbf{M}$. The diagonal element will be $\mathbf{M}_{\ell\ell}$ and the radius will be $\sum_{i\neq \ell} |\mathbf{M}_{\ell i}| = \sum_{i \neq \ell} \mathbf{M}_{\ell i}$ since all $\mathbf{M}_{\ell i} \geq 0$. This will be a circle with its center in $\mathbf{M}_{\ell \ell} \in [0,1]$, and a radius of $\sum_{i \neq \ell} \mathbf{M}_{\ell i} = 1-\mathbf{M}_{\ell \ell}$. So, this circle will have $1$ on its perimeter. This is true for all Gershgorin circles for this matrix (since $\ell$ was chosen arbitrarily). Thus, since all eigenvalues lie in the union of the Gershgorin circles, all eigenvalues $\lambda_i$ satisfy $|\lambda_i| \leq 1$. 

In the continuous case, we use similar arguments. Specifically, let $\lambda$ be an eigenvalue of $\calL$ with corresponding eigenfunction $f_\lambda$. Let $x^\star = \arg\max_{x\in\calX}|f_\lambda(x)|$. Because $f_\lambda$ is an eigenfunction, we have
\begin{align}
    (\calL f_\lambda)(x) = \bE_{Y\sim Q_{Y}}\pp{\calL(x,Y)f_\lambda(Y)} = \lambda\cdot f_\lambda(x),\label{eqn:GershCont}
\end{align}
for any $x\in\calX$. Substituting $x = x^\star $ in \eqref{eqn:GershCont} and isolating $\lambda$, we get
\begin{align}
    |\lambda| &= \frac{\abs{\bE_{Y\sim Q_{Y}}\pp{\calL(x,Y)f_\lambda(Y)}}}{|f_\lambda(x^\star)|}\\
    &\leq \bE_{Y\sim Q_{Y}}\pp{\calL(x,Y)\abs{\frac{f_\lambda(Y)}{f_\lambda(x^\star)}}}\\
    &\leq \bE_{Y\sim Q_{Y}}\pp{\calL(x,Y)} = 1,
\end{align}
where the first inequality follows from the triangle inequality, the second inequality is because $\abs{\frac{f_\lambda(x)}{f_\lambda(x^\star)}}\leq1$, for any $x\in\calX$, and the last equality follows by the definition of $\calL$. Since $\lambda$ is arbitrary, the above arguments hold for any eigenvalue of $\calL$. Finally, note that the identity eigenfunction, i.e., $f(x) = 1$, for any $x\in\calX$, corresponds to $\lambda=1$.

\end{proof}
\subsection{Auxiliary lemmata}

The following lemma concerns the approximation of the joint distribution of $k$-cycles by independent Poisson random variables.
\begin{lemma}\label{lem:PoisApp}\cite[Theorem 2]{arratia1992cycle} Let $1\leq k\leq n$ be an integer, and let $P_1\dots,P_k$ be independent random variables such that for all $1\leq i \leq k$, $P_i\sim \s{Poisson}\p{i^{-1}}$. Then, the total variation between the law of $N_1,\dots, N_k$ and $P_1,\dots,P_k$ satisfies
\begin{align}
d_{\s{TV}}\p{\pr\p{N_1,N_2,\ldots,N_k},\pr\p{P_1,P_2,\ldots,P_k}}\leq F\p{\frac{n}{k}},
\end{align}
where $F(x)$ is a monotone decreasing function satisfying $\log F(x)=-x\log x
(1+o(1))$ as $x\to \infty$.
\end{lemma}

In light of Lemma~\ref{lem:PoisApp}, let us bound \eqref{eqn:seocndMoment} when $N_k$ is replaced by $\s{Poisson}\p{i^{-1}}$, for all $1\leq k\leq n$. We have the following result.
\begin{lemma}\label{lem:PoisSecond}
Let $1\leq m\leq n$ be an integer, and let $P_1\dots,P_m$ be independent random variables such that for all $1\leq i \leq m$, $P_i\sim \s{Poisson}\p{i^{-1}}$. Then, for $d=1$,
\begin{align}
\E_{\sigma}\pp{\exp\p{\sum_{k=1}^mP_k\log g_k(Q_{XY},P_{XY})}}\leq \exp\pp{-\sum_{i\geq1}\log(1-\lambda_i^2)}\leq \exp\pp{\sum_{i\geq1}\frac{\lambda_i^2}{1-\lambda_i^2}}.
\end{align}
For any $d\geq1$,
\begin{align}
&\E_{\sigma}\pp{\exp\p{d\sum_{k=1}^mP_k\log g_k(Q_{XY},P_{XY})}}\nonumber\\
&\hspace{1cm} \leq \exp\pp{-d\sum_{i\geq1}\log(1-\lambda_i^2)+\p{\sum_{i\in\mathbb{N}}\lambda_i^{2}}^{d-2}\p{d\sum_{i\geq1}\log(1-\lambda_i^2)}^2}\\
&\hspace{1cm}\leq \exp\pp{d\sum_{i\geq1}\frac{\lambda_i^2}{1-\lambda_i^2}+\p{\sum_{i\in\mathbb{N}}\lambda_i^{2}}^{d-2}\p{d\sum_{i\geq1}\frac{\lambda_i^2}{1-\lambda_i^2}}^2}.
\end{align}
In particular, if \eqref{eqn:weakdetectlowerCond} holds, then,
\begin{align}
    \E_{\sigma}\pp{\exp\p{d\sum_{k=1}^mP_k\log g_k(Q_{XY},P_{XY})}}\leq 1+o(1).
\end{align}

\end{lemma}
\begin{proof}[Proof of Lemma~\ref{lem:PoisSecond}]
By the definition of the moment generating function of a Poisson
random variable, we get for $d=1$,
\begin{align}
\E_{\sigma}\pp{\exp\p{\sum_{k=1}^nP_k\log g_k(Q_{XY},P_{XY})}} &= \prod_{k=1}^n\E_{\sigma}\pp{e^{P_k\log g_k(Q_{XY},P_{XY})}}\\
& = \prod_{k=1}^n\exp\pp{\frac{1}{k}\p{e^{\log g_k(Q_{XY},P_{XY})}-1}}\\
& = \prod_{k=1}^n\exp\pp{\frac{1}{k}\p{g_k(Q_{XY},P_{XY})-1}}\\
& = \exp\pp{\sum_{k=1}^n\frac{1}{k}\p{\p{\sum_{i\in\mathbb{N}}\lambda_i^{2k}}-1}}\\
& \stackrel{\lambda_0=1}{=} \exp\pp{\sum_{k=1}^n\frac{1}{k}\p{\sum_{i\geq1}\lambda_i^{2k}}}\\
&\leq \exp\pp{\sum_{i\geq1}\sum_{k=1}^\infty\frac{\lambda_i^{2k}}{k}}\\
& = \exp\pp{-\sum_{i\geq1}\log(1-\lambda_i^2)}\\
&\leq \exp\pp{\sum_{i\geq1}\frac{\lambda_i^2}{1-\lambda_i^2}},
\end{align}
where in the last inequality we have used the fact that $\log(1+x)\geq\frac{x}{1+x}$, for any $x>-1$. For $d\geq1$, note that by Lagrange's remainder theorem and obtain that for all $x\in(0,1)$,
\begin{align}
    (1+x)^d &= 1+dx+\frac{d(d-1)}{2}(1+c)^{d-2}x^2\\
    &\leq 1+dx+\frac{d(d-1)}{2}(1+x)^{d-2}x^2,
\end{align}
where $c$ is a point in $[0, x]$. Thus, taking $x = \sum_{i\geq1}\lambda_i^{2k}$, we get,
\begin{align}
    \p{\sum_{i\in\mathbb{N}}\lambda_i^{2k}}^d 
    &\leq 1+d\sum_{i\geq1}\lambda_i^{2k}+\frac{d(d-1)}{2}\p{\sum_{i\in\mathbb{N}}\lambda_i^{2k}}^{d-2}\p{\sum_{i\geq1}\lambda_i^{2k}}^2\\
    &\leq1+d\sum_{i\geq1}\lambda_i^{2k}+\frac{d(d-1)}{2}\p{\sum_{i\in\mathbb{N}}\lambda_i^{2}}^{d-2}\p{\sum_{i\geq1}\lambda_i^{2k}}^2.\label{eqn:Largrange}
\end{align}
Then,
\begin{align}
&\E_{\sigma}\pp{\exp\p{d\sum_{k=1}^nP_k\log g_k(Q_{XY},P_{XY})}} \\
& = \prod_{k=1}^n\E_{\sigma}\pp{e^{dP_k\log g_k(Q_{XY},P_{XY})}}\\
& = \prod_{k=1}^n\exp\pp{\frac{1}{k}\p{e^{d\log g_k(Q_{XY},P_{XY})}-1}}\\
& = \prod_{k=1}^n\exp\pp{\frac{1}{k}\p{g_k^d(Q_{XY},P_{XY})-1}}\\
& = \exp\pp{\sum_{k=1}^n\frac{1}{k}\p{\p{\sum_{i\in\mathbb{N}}\lambda_i^{2k}}^d-1}}\\
& \stackrel{\eqref{eqn:Largrange}}{\leq} \exp\pp{\sum_{k=1}^n\frac{1}{k}\p{d\sum_{i\geq1}\lambda_i^{2k}+\frac{d(d-1)}{2}\p{\sum_{i\in\mathbb{N}}\lambda_i^{2}}^{d-2}\p{\sum_{i\geq1}\lambda_i^{2k}}^2}}\\
&\leq \exp\pp{d\sum_{i\geq1}\sum_{k=1}^\infty\frac{\lambda_i^{2k}}{k}+\frac{d(d-1)}{2}\p{\sum_{i\in\mathbb{N}}\lambda_i^{2}}^{d-2}\sum_{k=1}^n\frac{1}{k}\p{\sum_{i\geq1}\lambda_i^{2k}}^2}\\
& \leq \exp\pp{-d\sum_{i\geq1}\log(1-\lambda_i^2)+\frac{d(d-1)}{2}\p{\sum_{i\in\mathbb{N}}\lambda_i^{2}}^{d-2}\p{\sum_{i\geq1}\log(1-\lambda_i^2)}^2}\\
& \leq \exp\pp{-d\sum_{i\geq1}\log(1-\lambda_i^2)+\p{\sum_{i\in\mathbb{N}}\lambda_i^{2}}^{d-2}\p{d\sum_{i\geq1}\log(1-\lambda_i^2)}^2}\\
&\leq \exp\pp{d\sum_{i\geq1}\frac{\lambda_i^2}{1-\lambda_i^2}+\p{\sum_{i\in\mathbb{N}}\lambda_i^{2}}^{d-2}\p{d\sum_{i\geq1}\frac{\lambda_i^2}{1-\lambda_i^2}}^2}.
\end{align}
Under the condition in \eqref{eqn:weakdetectlowerCond}, we have
\begin{align}
    \p{\sum_{i\in\mathbb{N}}\lambda_i^{2}}^{d-2}&\leq \p{\sum_{i\in\mathbb{N}}\lambda_i^{2}}^d \\
    &= \p{1+\sum_{i\geq1}\lambda_i^{2}}^d\\
    &\leq \exp\p{d\sum_{i\geq1}\lambda_i^{2}}\\
    & = \exp(o(1)) = 1+o(1),
\end{align}
where we have used that fact that $\sum_{i\geq1}\lambda_i^{2}\leq \sum_{i\geq1}\frac{\lambda_i^{2}}{1-\lambda_i^2}$. Thus,
\begin{align}
    &\E_{\sigma}\pp{d\exp\p{\sum_{k=1}^nP_k\log g_k(Q_{XY},P_{XY})}}\nonumber\\
    &\hspace{2cm}\leq \exp\pp{d\sum_{i\geq1}\frac{\lambda_i^2}{1-\lambda_i^2}+\p{\sum_{i\in\mathbb{N}}\lambda_i^{2}}^{d-2}\p{d\sum_{i\geq1}\frac{\lambda_i^2}{1-\lambda_i^2}}^2}\\
    &\hspace{2cm} = \exp\p{o(1)+(1+o(1))o(1)}\\
    &\hspace{2cm}= 1+o(1).
\end{align}
\end{proof}
For simplicity of notation, we denote
\begin{align}
    \calB_{n,d}(\boldsymbol{\lambda}) \triangleq \exp\pp{d\sum_{i\geq1}\frac{\lambda_i^2}{1-\lambda_i^2}+\p{\sum_{i\in\mathbb{N}}\lambda_i^{2}}^{d-2}\p{d\sum_{i\geq1}\frac{\lambda_i^2}{1-\lambda_i^2}}^2}.
\end{align}
The following is the main ingredient in the proof of Theorems~\ref{th:LB} and \ref{th:LB2}. 
\begin{lemma}\label{lem:asymptoticReg}
    Let $N_k$ be the number of $k$-cycles in a uniformly distributed permutation $\sigma$, and $1\leq k \leq n$. Then: 
   \begin{enumerate}
   \item For all $(n,d)$ is holds that
       \begin{equation}
        \label{eq:bound-Dconst2}
        \E_{\sigma}\pp{\prod_{k=1}^n\p{\sum_{i\in\mathbb{N}}\lambda_i^{2k}}^{dN_k}}\leq\exp\pp{dn\sum_{i\geq1}\lambda_i^{2}}.
    \end{equation}
       \item If  at least one of $n,d$ tends to $\infty$ and \eqref{eqn:weakdetectlowerCond} holds, then
    \begin{equation}
        \label{eq:ProdBound} \E_{\sigma}\pp{\prod_{k=1}^n\p{\sum_{i\in\mathbb{N}}\lambda_i^{2k}}^{dN_k}}\leq 1+o(1).
    \end{equation}
    \item If both $n$, $d$ tends to $\infty$ and \eqref{eqn:strongdetectlowerCond2} holds, then 
       \begin{equation}
        \label{eq:bound-Dconst1}
        \E_{\sigma}\pp{\prod_{k=1}^n\p{\sum_{i\in\mathbb{N}}\lambda_i^{2k}}^{dN_k}}\leq (1+o(1))\cdot\calB_{n,d}(\boldsymbol{\lambda}).
    \end{equation}
    \item If $d$ is a constant, $\{\lambda_i\}_{i\geq1}$ are independent of $n$ satisfying \eqref{eqn:strongdetectlowerCond1}, and $n\to \infty$, then
    \begin{equation}
        \label{eq:bound-Dconst}
        \E_{\sigma}\pp{\prod_{k=1}^n\p{\sum_{i\in\mathbb{N}}\lambda_i^{2k}}^{dN_k}}\leq (1+o(1))\cdot\calB_{n,d}(\boldsymbol{\lambda}).
    \end{equation}
   \end{enumerate}
\end{lemma}

\begin{proof}[Proof of Lemma~\ref{lem:asymptoticReg}]
    
The idea of the proof is as follows: we consider the expectation of the product given in \eqref{eq:PermProd2}. In the case that $n\to\infty$, we show that the product of the last $n- \log n$ terms is always upper-bounded by $1+o(1)$. For the expectation of the product of the first $\log n$ terms, we use the Poisson approximation in Lemma~\ref{lem:PoisSecond}. The other case, where $n$ is constant, is solvable using elementary arguments. 

\paragraph{Proof of \eqref{eq:bound-Dconst2}.} Whenever $n$ is fixed we have $\sum_{k=1}^n k N_K=n$, which implies that for any $1\leq m \leq n$ we have
\begin{align}
\prod_{k=m}^n\p{\sum_{i\in\mathbb{N}}\lambda_i^{2k}}^{dN_k}&\leq\prod_{k=m}^n\p{\sum_{i\in\mathbb{N}}\lambda_i^{2m}}^{dN_k}\\
    &=\p{\sum_{i\in\mathbb{N}}\lambda_i^{2m}}^{d\sum_{k=m}^nN_k}\\
    &\leq \p{\sum_{i\in\mathbb{N}}\lambda_i^{2m}}^{d n}\\
    &\leq \exp\pp{dn\log \p{\sum_{i\in\mathbb{N}}\lambda_i^{2m}}}\\
    & = \exp\pp{dn\log \p{1+\sum_{i\geq1}\lambda_i^{2m}}}\\
    &\leq \exp\pp{dn\sum_{i\geq1}\lambda_i^{2m}}
    %&\leq \exp\pp{dn\p{\sum_{i\geq1}\lambda_i^{2}}^m}
    \label{eq:tailBoundGen},
\end{align}
and \eqref{eq:bound-Dconst2} corresponds to the case where $m=1$.

\paragraph{Proof of \eqref{eq:ProdBound}--\eqref{eq:bound-Dconst1}.} We assume $d,n\to\infty$, and either \eqref{eqn:weakdetectlowerCond} or \eqref{eqn:strongdetectlowerCond2}. Take $m=\ceil{\log(n)}$, and note that
\begin{align}
   nd\sum_{i\geq1}\lambda_i^{2m}&\leq nd\p{\sum_{i\geq1}\lambda_i^{2}}^m \\
   &\leq  nd\p{\sum_{i\geq1}\lambda_i^{2}}^{\log n} \\
   &= e\pp{d\sum_{i\geq1}\lambda_i^{2}}\p{\frac{n}{e}}^{1+\log\sum_{i\geq1}\lambda_i^{2}} \\
   &= o(1),
\end{align}
where the last passage clearly holds under either \eqref{eqn:weakdetectlowerCond} or \eqref{eqn:strongdetectlowerCond2}, and $d\to\infty$. This implies that
\begin{align}
\prod_{k=m}^n\p{\sum_{i\in\mathbb{N}}\lambda_i^{2k}}^{dN_k}&\leq\exp\pp{dn\sum_{i\geq1}\lambda_i^{2m}}\leq \exp(o(1)) = 1+o(1).\label{eq:tailBoundGen3}
\end{align}
We now evaluate the product of the first $m$ terms. For a fixed integer $m$, we consider the set $S_{n,m}\subseteq \N^m$ given by 
\begin{align}
S_{n,m}=\ppp{(n_1,\dots,n_m)\in \N^d; \sum_{k=1}^m n_k\leq n},
\end{align}
and a function $f_{n,m}:\N^m\to [0,\infty]$ given by 
\begin{align}
f_{n,m}(n_1,n_2,\dots,n_m)=\prod_{k=1}^m\p{\sum_{i\in\mathbb{N}}\lambda_i^{2k}}^{dn_k}\cdot \Ind_{S_{n,m}}(n_1,\dots,n_m).
\end{align}
We note that for all $n_1,\dots,n_m\in \N$,
\begin{equation}
    \label{eq:LBfBound2}
    f_{n,m}(n_1,n_2,\dots,n_m)\leq \prod_{k=1}^m\p{\sum_{i\in\mathbb{N}}\lambda_i^{2}}^{dn_k}\cdot \Ind_{S_{n,m}}(n_1,\dots,n_m)\leq \p{\sum_{i\in\mathbb{N}}\lambda_i^{2}}^{dn}.
\end{equation}
Let $\{P_k\}_k$ be independent $\s{Poisson}\p{k^{-1}}$ random variables as in Lemma~\ref{lem:PoisApp}.  Since $\sum_{k=1}^{m} N_k\leq n$ with probability $1$, we have 
\begin{align}
    \E_{\sigma}\pp{\prod_{k=1}^m\p{\sum_{i\in\mathbb{N}}\lambda_i^{2k}}^{dN_k}}&=\E_{\sigma}\pp{f_{n,m}(N_1,\dots,N_{m})}\\
     &\leq \E[f_{n,m}(P_1,\dots,P_m)]+d_{\s{TV}}\p{\pr\p{N_1^m},\pr\p{P_1^m}}\cdot \norm{f_{n,m}}_\infty \\
     &\leq \E_{\sigma}\pp{\prod_{k=1}^m\p{\sum_{i\in\mathbb{N}}\lambda_i^{2k}}^{dP_k}}+d_{\s{TV}}\p{\pr\p{N_1^m},\pr\p{P_1^m}}\cdot \norm{f_{n,m}}_\infty \\
     &\overset{(a)}{\leq} \E_{\sigma}\pp{\prod_{k=1}^m\p{\sum_{i\in\mathbb{N}}\lambda_i^{2k}}^{dP_k}}+F\p{\frac{n}{\ceil{\log n}}}\cdot \p{\sum_{i\in\mathbb{N}}\lambda_i^{2}}^{dn} \\
     & = \E_{\sigma}\pp{\prod_{k=1}^m\p{\sum_{i\in\mathbb{N}}\lambda_i^{2k}}^{dP_k}}+F\p{\frac{n}{m}}\cdot \p{\sum_{i\in\mathbb{N}}\lambda_i^{2}}^{dn},
      %&\overset{(b)}{\leq} \exp\pp{-\sum_{i\geq1}\log(1-\lambda_i^2)}+F\p{\frac{n}{m}}\cdot \p{\frac{1}{1-\rho^{2}}}^{dn}\\
      %&= \exp\pp{-\sum_{i\geq1}\log(1-\lambda_i^2)}+F\p{\frac{n}{\ceil{\alpha\log n}}}\cdot \p{\sum_{i\in\mathbb{N}}\lambda_i^{2}}^{dn},
\end{align}
where (a) follows from \eqref{eq:LBfBound2} and Lemma~\ref{lem:PoisApp}. By Lemma~\ref{lem:PoisApp},  we also have 
\begin{align}
    &\log\p{F\p{\frac{n}{\ceil{\log n}}}\cdot \p{\sum_{i\in\mathbb{N}}\lambda_i^{2}}^{dn}}\leq  -(1+o(1))\frac{n}{\log n}\log\p{\frac{n}{\log n}}+nd\log\sum_{i\in\mathbb{N}}\lambda_i^{2}\nonumber\\
    &\hspace{2cm}=-n\p{(1+o(1))\p{1-\frac{\log\log n }{ \log n}}-d\log\sum_{i\in\mathbb{N}}\lambda_i^{2}}\\
    &\hspace{2cm}=-n\pp{1-d\log\p{1+\sum_{i\geq1}\lambda_i^{2}}+o(1)}\\
    &\hspace{2cm}\overset{(a)}{\leq}-n\p{1-d\sum_{i\geq1}\lambda_i^{2}+o(1)}\\
    &\hspace{2cm}\xrightarrow[n\to \infty]{(b)} -\infty,
\end{align} 
where (a) follows from $\log (1+x)\leq x$, for $x\geq0$, and (b) holds for either \eqref{eqn:weakdetectlowerCond} or \eqref{eqn:strongdetectlowerCond2}. Thus,
\begin{align}
F\p{\frac{n}{\ceil{\log n}}}\cdot \p{\sum_{i\in\mathbb{N}}\lambda_i^{2}}^{dn}=o(1),
\end{align}
and therefore,
\begin{equation}\label{eq:LBprefBound2}
   \E_{\sigma}\pp{\prod_{k=1}^{\ceil{\log n}}\p{\sum_{i\in\mathbb{N}}\lambda_i^{2k}}^{dN_k}}\leq  \E_{\sigma}\pp{\prod_{k=1}^{\ceil{\log n}}\p{\sum_{i\in\mathbb{N}}\lambda_i^{2k}}^{dP_k}}+o(1).
\end{equation}
Combining \eqref{eq:tailBoundGen3} and \eqref{eq:LBprefBound2} together we conclude:
 \begin{align}
     \E_{\sigma}\pp{\prod_{k=1}^n\p{\sum_{i\in\mathbb{N}}\lambda_i^{2k}}^{dN_k}}&\leq \E_{\sigma}\pp{\prod_{k=1}^{\ceil{\log n}}\p{\sum_{i\in\mathbb{N}}\lambda_i^{2k}}^{dN_k}\prod_{k=\ceil{\log n}}^n\p{\sum_{i\in\mathbb{N}}\lambda_i^{2k}}^{dN_k}}\\
     &=(1+o(1))\cdot\E_{\sigma}\pp{\prod_{k=1}^{\ceil{\log n}}\p{\sum_{i\in\mathbb{N}}\lambda_i^{2k}}^{dN_k}}\\
     &\leq (1+o(1))\E_{\sigma}\pp{\prod_{k=1}^{\ceil{\log n}}\p{\sum_{i\in\mathbb{N}}\lambda_i^{2k}}^{dP_k}}+o(1)\\
     &\leq (1+o(1))\cdot\calB_{n,d}(\boldsymbol{\lambda})+o(1),\label{eqn:midCalc}
     %&\leq(1+o(1))\exp\pp{-\sum_{i\geq1}\log(1-\lambda_i^2)}+o(1).
 \end{align}
 where the last inequality follows from Lemma~\ref{lem:PoisSecond}. This proves \eqref{eq:bound-Dconst1}. Also, note that under \eqref{eqn:weakdetectlowerCond} the second part of Lemma~\ref{lem:PoisSecond} implies that $\calB_{n,d}(\boldsymbol{\lambda}) = 1+o(1)$, and combined with \eqref{eqn:midCalc} this proves \eqref{eq:ProdBound} when $n,d\to\infty$. When $n$ is fixed, and $d$ tends to infinity, by \eqref{eq:bound-Dconst2} we clearly have,
 \begin{equation}
        \E_{\sigma}\pp{\prod_{k=1}^n\p{\sum_{i\in\mathbb{N}}\lambda_i^{2k}}^{dN_k}}\leq\exp\pp{nd\sum_{i\geq1}\lambda_i^{2}} = 1+o(1),
    \end{equation}
under condition \eqref{eqn:weakdetectlowerCond}. Finally, the case where $d$ is fixed and $n$ tends to infinity follows from the arguments below for proving \eqref{eq:bound-Dconst}.
\paragraph{Proof of \eqref{eq:bound-Dconst}:} We repeat the same arguments as in the proof of \eqref{eq:ProdBound} when $n\to\infty$, with some minor changes. We assume that $d$ is fixed, and that $\{\lambda_i\}_{i\geq1}$ are a sequence of fixed real-valued numbers satisfying \eqref{eqn:strongdetectlowerCond1}. Take $m=\ceil{\alpha \log n}$, where $\alpha=-\frac{1}{\log\lambda_{\max}^2}+\varepsilon$, $\lambda_{\max}\triangleq\max_{i\geq1}\lambda_i$, and $\varepsilon$ sufficiently small so that
\begin{equation}
    \label{eq:epsChoice}
    -d\frac{\log\sum_{i\in\mathbb{N}}\lambda_i^{2}}{\log\lambda_{\max}^2}<\frac{1}{1-\varepsilon\log\lambda_{\max}^2},
\end{equation}
which is guaranteed to exist by \eqref{eqn:strongdetectlowerCond1}. We observe that 
\begin{align}
    nd\sum_{i\geq1}\lambda_i^{2m} &\leq  d\sum_{i\geq1}n\lambda_i^{2\alpha \log n} \\
    &= d\sum_{i\geq1}n^{1+\alpha\log\lambda_i^{2}}\\
    &= d\sum_{i\geq1}n^{1-\frac{\log\lambda_i^{2}}{\log\lambda_{\max}^2}+\varepsilon\log\lambda_i^{2}}\\
    &\leq d\sum_{i\geq1}n^{\varepsilon\log\lambda_i^{2}}\\
    & = d\sum_{i\geq1}\lambda_i^{2\varepsilon\log n}\\
    &\leq d\p{\sum_{i\geq1}\lambda_i^{2}}^{\varepsilon\log n}\\
    & = o(1).
\end{align}
%where the last passage hold if $\sum_{i\geq1}\lambda_i^{2}<1$. 
This implies that
\begin{align}
\prod_{k=m}^n\p{\sum_{i\in\mathbb{N}}\lambda_i^{2k}}^{dN_k}&\leq\exp\pp{dn\sum_{i\geq1}\lambda_i^{2m}}\leq \exp(o(1)) = 1+o(1).\label{eq:tailBoundGen2}
\end{align}
We now evaluate the product of the first $m$ terms. In a similar fashion to the first part, by Lemmas~\ref{lem:PoisApp} and \ref{lem:PoisSecond},
\begin{align}
    \E_{\sigma}\pp{\prod_{k=1}^m\p{\sum_{i\in\mathbb{N}}\lambda_i^{2k}}^{dN_k}}&\leq \E_{\sigma}\pp{\prod_{k=1}^m\p{\sum_{i\in\mathbb{N}}\lambda_i^{2k}}^{dP_k}}+F\p{\frac{n}{\ceil{\alpha\log n}}}\cdot \p{\sum_{i\in\mathbb{N}}\lambda_i^{2}}^{dn}.
\end{align}
By Lemma~\ref{lem:PoisApp},  we also have 
\begin{align}
    &\log\p{F\p{\frac{n}{\ceil{\alpha\log n}}}\cdot \p{\sum_{i\in\mathbb{N}}\lambda_i^{2}}^{dn}}\leq  -(1+o(1))\frac{n}{\alpha\log n}\log\p{\frac{n}{\alpha\log n}}+nd\log\sum_{i\in\mathbb{N}}\lambda_i^{2}\nonumber\\
    &\hspace{2cm}=-n\p{(1+o(1))\frac{1}{\alpha}\p{1-\frac{\log (\alpha\log n) }{ \log n}}-d\log\sum_{i\in\mathbb{N}}\lambda_i^{2}}\\
    &\hspace{2cm}=-n(1+o(1))\p{\frac{1}{\alpha}-d\log\sum_{i\in\mathbb{N}}\lambda_i^{2}}\\
    &\hspace{2cm}=-n(1+o(1))\p{\frac{\log\lambda_{\max}^2}{\varepsilon\log\lambda_{\max}^2-1}-d\log\sum_{i\in\mathbb{N}}\lambda_i^{2}}.
\end{align} 
By \eqref{eq:epsChoice} we have
\begin{align}
    \frac{\log\lambda_{\max}^2}{\varepsilon\log\lambda_{\max}^2-1}-d\log\sum_{i\in\mathbb{N}}\lambda_i^{2}>0,
\end{align}
which implies that 
\begin{align}
F\p{\frac{n}{\ceil{\alpha\log n}}}\cdot \p{\sum_{i\in\mathbb{N}}\lambda_i^{2}}^{dn}=o(1),
\end{align}
 and therefore,
\begin{equation}\label{eq:LBprefBound}
   \E_{\sigma}\pp{\prod_{k=1}^m\p{\sum_{i\in\mathbb{N}}\lambda_i^{2k}}^{dN_k}}\leq  \E_{\sigma}\pp{\prod_{k=1}^m\p{\sum_{i\in\mathbb{N}}\lambda_i^{2k}}^{dP_k}}+o(1).
\end{equation}
Combining \eqref{eq:tailBoundGen2} and \eqref{eq:LBprefBound} together we conclude:
 \begin{align}
     \E_{\sigma}\pp{\prod_{k=1}^n\p{\sum_{i\in\mathbb{N}}\lambda_i^{2k}}^{dN_k}}&\leq \E_{\sigma}\pp{\prod_{k=1}^{\ceil{\alpha\log n}}\p{\sum_{i\in\mathbb{N}}\lambda_i^{2k}}^{dN_k}\prod_{k=\ceil{\alpha\log n}}^n\p{\sum_{i\in\mathbb{N}}\lambda_i^{2k}}^{dN_k}}\\
     &=(1+o(1))\cdot\E_{\sigma}\pp{\prod_{k=1}^{\ceil{\alpha\log n}}\p{\sum_{i\in\mathbb{N}}\lambda_i^{2k}}^{dN_k}}\\
     &\leq (1+o(1))\E_{\sigma}\pp{\prod_{k=1}^m\p{\sum_{i\in\mathbb{N}}\lambda_i^{2k}}^{dP_k}}+o(1)\\
     &\leq (1+o(1))\cdot\calB_{n,d}(\boldsymbol{\lambda})+o(1),
 \end{align}
 where the last inequality follows from Lemma~\ref{lem:PoisSecond}. This proves \eqref{eq:bound-Dconst}.

%\begin{align}
%     \boxed{\s{trace}\p{\calL^2\cdot(\mathbf{I}-\calL^2)^{-1}}=o(d^{-1}).}
% \end{align}
%For fixed $d$ and $n\to\infty$, strong recovery is impossible if
%\begin{align}
%     \boxed{d<\frac{\log\sum_{i\geq1}\lambda_i^{2}}{\log\sum_{i\in\mathbb{N}}\lambda_i^{2}},}
%\end{align}
%which can be rewritten as,
%\begin{align}
%     \boxed{d<\frac{\log(\norm{\boldsymbol{\lambda}}_{\ell_2}^2-1)}{\log\norm{\boldsymbol{\lambda}}_{\ell_2}^2} = \frac{\log(\s{trace}(\calL^2)-1)}{\log\s{trace}(\calL^2)}.}
%\end{align}

\end{proof}

\subsection{Proof of Theorem~\ref{th:LB}}

Combining \eqref{eq:LikeBound2}, \eqref{eq:PermProd2}, and \eqref{eq:ProdBound} in Lemma~\ref{lem:asymptoticReg}, we obtain,
\begin{align}
    {\s{R}}^\star&\geq 1-\sqrt{\E_{\calH_0}[\s{L}_n^2]-1}\\
    &=1-\sqrt{\E_{\sigma}\pp{\prod_{k=1}^n\p{\sum_{i\in\mathbb{N}}\lambda_i^{2k}}^{dN_k}}-1}\\
    &\geq 1-\sqrt{1+o(1)-1}\\
    &=1+o(1),
\end{align}
which concludes the proof.

\subsection{Proof of Theorem~\ref{th:LB2}}

Let $(\calD,d,n)$ be a sequence satisfying the assumptions of Theorem~\ref{th:LB2}. Recall that this sequence is not admissible for strong detection if, 
\begin{align}
    \bE_{\calH_0}[\s{L}_n^2] = O(1).
\end{align}
%which implies, 
%\begin{equation}\label{eq:ImpStrongDet}
%    \s{R}^\star\geq 1-d_{\s{TV}}(\P_{\calH_0},\P_{\calH_1})=\Omega(1),
%\end{equation}
%and thus $(\calD,d,n)$ is not admissible for strong detection (as \eqref{eq:ImpStrongDet} implies $\limsup\s{R}^\star>0$). 
Now, by Lemma~\ref{lem:asymptoticReg}, if $d$ is a constant, $\{\lambda_i\}_{i\geq1}$ are independent of $n$ satisfying \eqref{eqn:strongdetectlowerCond1}, and $n\to \infty$, then
    \begin{equation}
        \bE_{\calH_0}[\s{L}_n^2]\leq (1+o(1))\cdot\calB_{n,d}(\boldsymbol{\lambda}) = O(1).
    \end{equation}
On the other hand, if both $n$ and $d$ tend to $\infty$, and \eqref{eqn:strongdetectlowerCond2} holds, then $d\sum_{i\geq1}\frac{\lambda_i^2}{1-\lambda_i^2}<1$, and
\begin{align}
    \p{\sum_{i\in\mathbb{N}}\lambda_i^{2}}^{d-2}&\leq \p{\sum_{i\in\mathbb{N}}\lambda_i^{2}}^d \\
    &= \p{1+\sum_{i\geq1}\lambda_i^{2}}^d\\
    &\leq \exp\p{d\sum_{i\geq1}\lambda_i^{2}}\leq e,
\end{align}
where the last inequality follows from the fact that $d\sum_{i\geq1}\lambda_i^2\leq d\sum_{i\geq1}\frac{\lambda_i^2}{1-\lambda_i^2}<1$. Therefore,
       \begin{align}
        \bE_{\calH_0}[\s{L}_n^2]&\leq(1+o(1))\cdot\calB_{n,d}(\boldsymbol{\lambda})\\
        & = (1+o(1))\cdot \exp\pp{d\sum_{i\geq1}\frac{\lambda_i^2}{1-\lambda_i^2}+\p{\sum_{i\in\mathbb{N}}\lambda_i^{2}}^{d-2}\p{d\sum_{i\geq1}\frac{\lambda_i^2}{1-\lambda_i^2}}^2}\\
        &\leq (1+o(1))\cdot \exp\pp{1+e} = O(1).
    \end{align}
Finally, when $d\to\infty$, $n$ is constant, and \eqref{eqn:strongdetectlowerCond3} holds, by \eqref{eq:bound-Dconst2}, we have,
\begin{equation}        
\E_{\calH_0}[\s{L}_n^2]\leq\exp\pp{dn\sum_{i\geq1}\lambda_i^{2}} =O(1).
    \end{equation}
That concludes the proof.

\section{Proofs of Upper Bounds}

\subsection{Proof of Theorem~\ref{thm:2}}
Recall the test in \eqref{eqn:testscanmain}, and let us analyze its Type-I error probability. Let $\lambda\geq0$. We have,
\begin{align}
    \calE_1&\triangleq\pr_{\calH_0}\pp{\phi_{\s{GLRT}}(\s{X},\s{Y})=1} \\
    &= \pr_{\calH_0}\pp{\max_{\pi\in\mathbb{S}_n}\sum_{i=1}^n\log\frac{P_{XY}^{\otimes d}(X_i,Y_{\pi_i})}{Q^{\otimes d}_{XY}(X_i,Y_{\pi_i})}\geq dn\cdot\tau_{\s{GLRT}}}\\
    & = \pr_{\calH_0}\pp{\bigcup_{\pi\in\mathbb{S}_n}\sum_{i=1}^n\log\frac{P^{\otimes d}_{XY}(X_i,Y_{\pi_i})}{Q^{\otimes d}_{XY}(X_i,Y_{\pi_i})}\geq dn\cdot\tau_{\s{GLRT}}}\\
    &\leq n!\cdot\pr_{\calH_0}\pp{\sum_{i=1}^n\log\frac{P^{\otimes d}_{XY}(X_i,Y_{i})}{Q^{\otimes d}_{XY}(X_i,Y_{i})}\geq dn\cdot\tau_{\s{GLRT}}}\\
    &\leq n!\cdot e^{-dn\cdot\lambda\tau_{\s{GLRT}}}\bE_{\calH_0}\pp{\exp\p{\lambda\sum_{i=1}^n\log\frac{P^{\otimes d}_{XY}(X_i,Y_{\pi_i})}{Q^{\otimes d}_{XY}(X_i,Y_{\pi_i})}}}\\
    &= n!\cdot e^{-dn\cdot\lambda\tau_{\s{GLRT}}}\pp{\bE_{\calH_0}\pp{\exp\p{\lambda\log\frac{P^{\otimes d}_{XY}(X_1,Y_{1})}{Q^{\otimes d}_{XY}(X_1,Y_{1})}}}}^n\\
    & = n!\cdot e^{-dn\cdot\lambda\tau_{\s{GLRT}}}\pp{\bE_{Q_{XY}}\pp{\exp\p{\lambda\log\frac{P_{XY}(A,B)}{Q_{XY}(A,B)}}}}^{dn}\\
    &= \exp\pp{\log n!-dn\cdot\lambda\tau_{\s{GLRT}}+dn\cdot\psi_Q(\lambda)},
\end{align}
where the first inequality follows from the union bound and $|\mathbb{S}_n|=n!$, and the second inequality is by Chernoff's bound. For $\tau_{\s{GLRT}}\in(-d_{\s{KL}}(Q_{XY}||P_{XY}),d_{\s{KL}}(P_{XY}||Q_{XY}))$, we may take $\lambda\geq0$ so that $\lambda\tau_{\s{GLRT}}-\psi_Q(\lambda)$ is arbitrarily close to $E_Q(\tau_{\s{GLRT}})$. This implies that
\begin{align}
  \calE_1\leq \exp\pp{n\log \frac{n}{e}+\log n+1-dn\cdot E_Q(\tau_{\s{GLRT}})},  
\end{align}
where we have used Stirling approximation that $n!\leq en^{n+1}e^{-n}$. Therefore, we see that if,
\begin{align}
    E_Q(\tau_{\s{GLRT}})\geq \frac{\log (n/e)}{d}+\frac{1+\log n}{dn}+\omega(d^{-1}n^{-1}),\label{eqn:condScan}
\end{align}
%\begin{align}
%   \frac{\log (n/e)}{d}+\frac{1+\log n}{dn}<E_Q(\tau_{\s{GLRT}}),\label{eqn:condScan}
%\end{align}
then the Type-I error probability goes to zero. 
%Next, we take $\tau_{\s{GLRT}}\triangleq d\cdot n\cdot (1-\delta_n)\cdot d_{\s{KL}}(P_{XY}||Q_{XY})$, for some $\delta_n>0$, and $\lambda = \frac{\tau_{\s{GLRT}}}{nd}$. Then, we get 
%\begin{align}
%    \calE_1 &\leq\exp\pp{\log n!-\lambda\tau_{\s{GLRT}}+dn\log\bE_{Q_{XY}}\p{e^{\lambda\log\frac{P_{XY}(A,B)}{Q_{XY}(A,B)}}}}\\
%    & = \exp\pp{\log n!-dn(1-\delta_n)^2\cdot d_{\s{KL}}^2(P_{XY}||Q_{XY})+dn\log\bE_{Q_{XY}}\p{e^{\lambda\log\frac{P_{XY}(A,B)}{Q_{XY}(A,B)}}}}.
%\end{align}
%Thus, using Stirling's approximation, we see that if,
%\begin{align}
   %\frac{\log n}{d}<(1-\delta_n)^2d_{\s{KL}}^2(P_{XY}||Q_{XY})-\log\bE_{Q_{XY}}\pp{\frac{P_{XY}(A,B)}{Q_{XY}(A,B)}}^{d_{\s{KL}}(P_{XY}||Q_{XY})},\label{eqn:condScan}
%\end{align}
%then the Type-I error probability goes to zero. 
Next, we bound the Type-II error probability as follows. Under $\calH_1$, since our proposed test is invariant to reordering of $\s{X}$ and $\s{Y}$, we may assume without loss of generality that the latent permutation is the identity one, i.e., $\sigma=\s{Id}$. Then, for $\lambda\leq0$,
\begin{align}
    \calE_2&\triangleq\pr_{\calH_1}\pp{\phi_{\s{GLRT}}(\s{X},\s{Y})=0}\\
    &= \pr_{\calH_1}\pp{\max_{\pi\in\mathbb{S}_n}\sum_{i=1}^n\log\frac{P^{\otimes d}_{XY}(X_i,Y_{\pi_i})}{Q^{\otimes d}_{XY}(X_i,Y_{\pi_i})}<dn\cdot\tau_{\s{GLRT}}}\\
    & = \pr_{\calH_1}\pp{\sum_{i=1}^n\log\frac{P^{\otimes d}_{XY}(X_i,Y_{i})}{Q^{\otimes d}_{XY}(X_i,Y_{i})}<dn\cdot\tau_{\s{GLRT}}}\\
    &\leq e^{-dn\cdot\lambda\tau_{\s{GLRT}}}\bE_{\calH_1}\pp{\exp\p{\lambda\sum_{i=1}^n\log\frac{P^{\otimes d}_{XY}(X_i,Y_{i})}{Q^{\otimes d}_{XY}(X_i,Y_{i})}}}\\
    & = e^{-dn\cdot\lambda\tau_{\s{GLRT}}}\pp{\bE_{\calH_1}\pp{\exp\p{\lambda\log\frac{P^{\otimes d}_{XY}(A,B)}{Q^{\otimes d}_{XY}(A,B)}}}}^n\\
    & = e^{-dn\cdot\lambda\tau_{\s{GLRT}}}\pp{\bE_{P_{XY}}\pp{\exp\p{\lambda\log\frac{P_{XY}(A,B)}{Q_{XY}(A,B)}}}}^{dn}\\
    & = \exp\pp{dn\cdot\psi_P(\lambda)-dn\cdot\lambda\tau_{\s{GLRT}}}.
\end{align}
Again, because $\tau_{\s{GLRT}}\in(-d_{\s{KL}}(Q_{XY}||P_{XY}),d_{\s{KL}}(P_{XY}||Q_{XY}))$, we may take $\lambda\leq0$ so that $\lambda\tau_{\s{GLRT}}-\psi_P(\lambda)$ is arbitrarily close to $E_P(\tau_{\s{GLRT}})$. Therefore,
\begin{align}
    \calE_2\leq\exp\p{-dn\cdot E_P(\tau_{\s{GLRT}})},
\end{align}
and thus if,
\begin{align}
    E_P(\tau_{\s{GLRT}})=\omega(d^{-1}n^{-1}),
\end{align}
then the Type-II error probability goes to zero. 

%As for the Type-II error probability, we get
%\begin{align}
%    &\calE_2 \leq \exp\pp{dn(1-\delta_n)^2\cdot d_{\s{KL}}^2(P_{XY}||Q_{XY})+dn\log\bE_{P_{XY}}\p{e^{-\lambda\log\frac{P_{XY}(A,B)}{Q_{XY}(A,B)}}}},
%\end{align}
%and thus if,
%\begin{align}
 %   (1-\delta_n)^2d_{\s{KL}}^2(P_{XY}||Q_{XY})+\log\bE_{P_{XY}}\pp{\frac{Q_{XY}(A,B)}{P_{XY}(A,B)}}^{d_{\s{KL}}(P_{XY}||Q_{XY})}<0,
%\end{align}
%then the Type-II error probability goes to zero.

%In fact, consider the following simpler analysis for the Type-II error probability, from which it is easier to understand the conditions under which it converges to zero. By Chebyshev's inequality,
%\begin{align}
%    \calE_2 &= \pr_{\calH_1}\pp{\max_{\pi\in\mathbb{S}_n}\sum_{i=1}^n\log\frac{P^{\otimes d}_{XY}(X_i,Y_{\pi_i})}{Q^{\otimes d}_{XY}(X_i,Y_{\pi_i})}<\tau_{\s{GLRT}}}\\
%    & = \pr_{\calH_1}\pp{\sum_{i=1}^n\log\frac{P^{\otimes d}_{XY}(X_i,Y_{i})}{Q^{\otimes d}_{XY}(X_i,Y_{i})}<\tau_{\s{GLRT}}}\\
%    &\leq \frac{\s{Var}_{P_{XY}}\p{\sum_{i=1}^n\log\frac{P^{\otimes d}_{XY}(X_i,Y_{i})}{Q^{\otimes d}_{XY}(X_i,Y_{i})}}}{(1-\delta_n)^2d^2n^2d^2_{\s{KL}}(P_{XY}||Q_{XY})}\\
%    & = \frac{nd\cdot\s{Var}_{P_{XY}}\p{\log\frac{P_{XY}(A,B)}{Q_{XY}(A,B)}}}{(1-\delta_n)^2d^2n^2d^2_{\s{KL}}(P_{XY}||Q_{XY})}\\
%    &=\frac{1}{(1-\delta_n)^2nd}\frac{\s{Var}_{P_{XY}}\p{\log\frac{P_{XY}(A,B)}{Q_{XY}(A,B)}}}{d^2_{\s{KL}}(P_{XY}||Q_{XY})},
%\end{align}
%which clearly goes to zero if \eqref{eqn:condScan} holds. 

\subsection{Proof of Theorem~\ref{thm:3}}
Consider the sum test in \eqref{eqn:testSummain} and \eqref{eqn:calkDef}, and let us analyze its Type-I error probability. By Chebyshev's inequality,
\begin{align}
    \pr_{\calH_0}\pp{\phi_{\s{sum}}(\s{X},\s{Y})=1} &= \pr_{\calH_0}\pp{\sum_{i,j=1}^n\sum_{\ell=1}^d\calK(X_{i\ell},Y_{j\ell})\geq\tau_{\s{sum}}}\\
    &\leq \frac{\s{Var}_{Q_{XY}}\p{\sum_{i,j=1}^n\sum_{\ell=1}^d\calK(X_{i\ell},Y_{j\ell})}}{\tau_{\s{sum}}^2},\label{eqn:TypeIChebyshev}
\end{align}
where we have used the fact that $\bE_{\calH_0}[\calK(X_{i\ell},Y_{j\ell})]=0$, for any $i,j\in[n]$ and $\ell\in[d]$. Note that,
\begin{align}
  \s{Var}_{Q_{XY}}\p{\sum_{i,j=1}^n\sum_{\ell=1}^d\calK(X_{i\ell},Y_{j\ell})} = d\cdot \s{Var}_{Q_{XY}}\p{\sum_{i,j=1}^n\calK(A_i,B_j)},\label{eqn:VariacnceTerm0}
\end{align}
where each $\{A_i\}_{i\geq1}$ and $\{B_i\}_{i\geq1}$ are i.i.d. random variables distributed according to $P_X$. Now, we have,
\begin{align}
\s{Var}_{Q_{XY}}\p{\sum_{i,j=1}^n\calK(A_i,B_j)}&=\sum_{i=1}^n\s{Var}_{Q_{XY}}\p{\sum_{j=1}^n\calK(A_i,B_j)}\nonumber\\
&\hspace{0.5cm}+ \sum_{i\neq k}\s{cov}_{Q_{XY}}\p{\sum_{j=1}^n\calK(A_i,B_j),\sum_{m=1}^n\calK(A_k,B_m)}.\label{eqn:varaCalc}
\end{align}
Let us start with the first term at the right-hand-side of \eqref{eqn:varaCalc}. We get,
\begin{align}
    \s{Var}_{Q_{XY}}\p{\sum_{j=1}^n\calK(A_i,B_j)} &= \sum_{j=1}^n\s{Var}_{Q_{XY}}\p{\calK(A_i,B_j)}\nonumber\\
    &\hspace{1cm}+\sum_{j\neq m}\s{cov}_{Q_{XY}}\p{\calK(A_i,B_j),\calK(A_i,B_m)}\\
    & = n\cdot \s{Var}_{Q_{XY}}\p{\calK(A,B)}\nonumber\\
    &\hspace{1cm}+n(n-1)\cdot\s{cov}_{Q_{XY}}\p{\calK(A,B),\calK(A,B')},\label{eqn:VariacnceTerm}
\end{align}
where $B$ and $B'$ are independent copies. Now, we note that
\begin{align}
    \s{cov}_{Q_{XY}}\p{\calK(A,B),\calK(A,B')} &=\bE_{Q_{XY}}\p{\calK(A,B)\cdot\calK(A,B')}-\bE_{Q_{XY}}^2\p{\calK(A,B)}.\label{eqn:Cov1zero}
\end{align}
By the definition of $\calK$ in \eqref{eqn:calkDef}, we have
\begin{align}
    \bE_{Q_{XY}}(\calK(A,B))&=\bE_{Q_{XY}}\pp{\log\frac{P_{XY}(A,B)}{Q_{XY}(A,B)}}-\bE_{Q_{XY}}\pp{\log\frac{P_{XY}(A,B)}{Q_{XY}(A,B)}}\nonumber\\
    &\quad-\bE_{Q_{XY}}\pp{\log\frac{P_{XY}(A,B)}{Q_{XY}(A,B)}}\quad-d_{\s{KL}}(Q_{XY}||P_{XY})\\
    & = -d_{\s{KL}}(Q_{XY}||P_{XY})+d_{\s{KL}}(Q_{XY}||P_{XY})+d_{\s{KL}}(Q_{XY}||P_{XY})-d_{\s{KL}}(Q_{XY}||P_{XY})\nonumber\\
    & = 0,
\end{align}
where we have used the fact that $(A,B,B')$ are statistically independent, and thus $(A,B,B')\sim P_X^{\otimes 3}$. On the other hand, by the law of total expectation, 
\begin{align}
    \bE_{Q_{XY}}\p{\calK(A,B)\cdot\calK(A,B')} = \bE_{Q_{XY}}\p{\calK(A,B)\cdot\bE_{Q_{XY}}[\calK(A,B')\vert A]}.
\end{align}
Now, we note that
\begin{align}
    \bE_{Q_{XY}}[\calK(A,B')\vert A] &= \bE_{Q_{XY}}\pp{\left.\log\frac{P_{XY}(A,B')}{Q_{XY}(A,B')}\right|A}-\bE_{Q_{XY}}\pp{\left.\log\frac{P_{XY}(A,B')}{Q_{XY}(A,B')}\right|A}\nonumber\\
    &\quad-\bE_{Q_{XY}}\pp{\log\frac{P_{XY}(A,B')}{Q_{XY}(A,B')}}\quad-d_{\s{KL}}(Q_{XY}||P_{XY})\\
    & = 0,
\end{align}
where again we have used the fact that $(A,B,B')$ are statistically independent. The above implies that $\bE_{Q_{XY}}\p{\calK(A,B)\cdot\calK(A,B')}=0$, and thus,
\begin{align}
    \s{cov}_{Q_{XY}}\p{\calK(A,B),\calK(A,B')}=0.\label{eqn:covzero}
\end{align}
Therefore, 
\begin{align}
    \s{Var}_{Q_{XY}}\p{\sum_{j=1}^n\calK(A_i,B_j)} = n\cdot \s{Var}_{Q_{XY}}\p{\calK(A,B)}.\label{eqn:VarVar}
\end{align}

Next, consider the second term at the right-hand-side of \eqref{eqn:varaCalc}. For $i\neq k$, we have
\begin{align}
    \s{cov}_{Q_{XY}}\p{\sum_{j=1}^n\calK(A_i,B_j),\sum_{m=1}^n\calK(A_k,B_m)}=\sum_{j=1}^n\sum_{m=1}^n\s{cov}_{Q_{XY}}\p{\calK(A_i,B_j),\calK(A_k,B_m)},\label{eqn:Cov2zero0}
\end{align}
By symmetry, using the exact same arguments we used to prove \eqref{eqn:covzero}, it is clear that
\begin{align}
\sum_{j=1}^n\s{cov}_{Q_{XY}}\p{\calK(A_i,B_j),\calK(A_k,B_j)} = 0.
\end{align}
Similarly,
\begin{align}
   \sum_{j\neq m}\s{cov}_{Q_{XY}}\p{\calK(A_i,B_j),\calK(A_k,B_m)} =0,
\end{align}
because $\calK(A_i,B_j)$ and $\calK(A_k,B_m)$ are statistically independent for $i\neq k$ and $j\neq m$. Thus,
\begin{align}
    \s{cov}_{Q_{XY}}\p{\sum_{j=1}^n\calK(A_i,B_j),\sum_{m=1}^n\calK(A_k,B_m)}=0.\label{eqn:Cov2zero}
\end{align}
Combining \eqref{eqn:VariacnceTerm0}, \eqref{eqn:varaCalc}, \eqref{eqn:VarVar}, and \eqref{eqn:Cov2zero}, we get,
\begin{align}
  \s{Var}_{Q_{XY}}\p{\sum_{i,j=1}^n\sum_{\ell=1}^d\calK(X_{i\ell},Y_{j\ell})} = dn^2\cdot\s{Var}_{Q_{XY}}\p{\calK(A,B)}.
\end{align}
Taking $\tau_{\s{sum}} = dn\frac{d_{\s{KL}}(P_{XY}||Q_{XY})+d_{\s{KL}}(Q_{XY}||P_{XY})}{2}$, we obtain from \eqref{eqn:TypeIChebyshev},
\begin{align}
    \pr_{\calH_0}\pp{\phi_{\s{sum}}(\s{X},\s{Y})=1} &\leq \frac{4\cdot\s{Var}_{Q_{XY}}\p{\calK(A,B)}}{d\cdot\p{d_{\s{KL}}(P_{XY}||Q_{XY})+d_{\s{KL}}(Q_{XY}||P_{XY})}^2}.\label{eqn:TypeIChebyshev2}
\end{align}

Next, we analyze the Type-II error probability. As in the GLRT analysis, under $\calH_1$, since our proposed test is invariant to reordering of $\s{X}$ and $\s{Y}$, we may assume without loss of generality that the latent permutation is the identity one, i.e., $\sigma=\s{Id}$. Then, by Chebyshev's inequality,
\begin{align}
    \pr_{\calH_1}\pp{\phi_{\s{sum}}(\s{X},\s{Y})=0} &= \pr_{\calH_1}\pp{\sum_{i,j=1}^n\sum_{\ell=1}^d\calK(X_{i\ell},Y_{j\ell})<\tau_{\s{sum}}}\\
    &\leq \frac{\s{Var}_{P_{XY}}\p{\sum_{i,j=1}^n\sum_{\ell=1}^d\calK(X_{i\ell},Y_{j\ell})}}{\pp{d n\cdot (d_{\s{KL}}(P_{XY}||Q_{XY})+d_{\s{KL}}(Q_{XY}||P_{XY}))-\tau_{\s{sum}}}^2}\\
    & = \frac{4\cdot\s{Var}_{P_{XY}}\p{\sum_{i,j=1}^n\sum_{\ell=1}^d\calK(X_{i\ell},Y_{j\ell})}}{d^2n^2\cdot\p{d_{\s{KL}}(P_{XY}||Q_{XY})+d_{\s{KL}}(Q_{XY}||P_{XY})}^2}.\label{eqn:TypeIChebyshev3}
\end{align}
As before, let us find the variance term in \eqref{eqn:TypeIChebyshev3}. Note that,
\begin{align}
  \s{Var}_{P_{XY}}\p{\sum_{i,j=1}^n\sum_{\ell=1}^d\calK(X_{i\ell},Y_{j\ell})} = d\cdot \s{Var}_{P_{XY}}\p{\sum_{i,j=1}^n\calK(A_i,B_j)},\label{eqn:VariacnceTerm01}
\end{align}
where each $\{(A_i,B_i)\}_{i=1}^n$ are i.i.d. pairs of random variables distributed according to $P_{XY}$. Now, we have,
\begin{align}
\s{Var}_{P_{XY}}\p{\sum_{i,j=1}^n\calK(A_i,B_j)}&=\sum_{i=1}^n\s{Var}_{P_{XY}}\p{\sum_{j=1}^n\calK(A_i,B_j)}\nonumber\\
&\hspace{0.5cm}+ \sum_{i\neq k}\s{cov}_{P_{XY}}\p{\sum_{j=1}^n\calK(A_i,B_j),\sum_{m=1}^n\calK(A_k,B_m)}.\label{eqn:varaCalcType2}
\end{align}
Let us start with the first term at the right-hand-side of \eqref{eqn:varaCalcType2}. We get,
\begin{align}
    \s{Var}_{P_{XY}}\p{\sum_{j=1}^n\calK(A_i,B_j)} &= \sum_{j=1}^n\s{Var}_{P_{XY}}\p{\calK(A_i,B_j)}\nonumber\\
    &\hspace{1cm}+\sum_{j\neq m}\s{cov}_{P_{XY}}\p{\calK(A_i,B_j),\calK(A_i,B_m)}\\
    & = \s{Var}_{P_{XY}}\p{\calK(A,B)}+(n-1)\cdot \s{Var}_{Q_{XY}}\p{\calK(A,B)}\nonumber\\
    &\hspace{1cm}+\sum_{j\neq m}\s{cov}_{P_{XY}}\p{\calK(A_i,B_j),\calK(A_i,B_m)}.\label{eqn:VariacnceTermType2temp}
\end{align}
Let us find $\s{cov}_{P_{XY}}\p{\calK(A_i,B_j),\calK(A_i,B_m)}$, for the different possible values of $(i,j,m)$. Without loss of generality, we fix $i=1$. For $j=1$ and $m\neq 1$, we have
\begin{align}
   \s{cov}_{P_{XY}}\p{\calK(A_1,B_1),\calK(A_1,B_m)} &= \bE_{P_{XY}}\pp{\calK(A_1,B_1)\cdot\calK(A_1,B_m)}\\
   & = \bE_{P_{XY}}\pp{\calK(A_1,B_1)\bE_{Q_{XY}}[\calK(A_1,B_m)\vert A_1]}\\
   & = 0,
\end{align}
where we have used the fact that $A_1$ and $B_m$ are statistically independent and thus $\bE_{Q_{XY}}[\calK(A_1,B_m)\vert A_1]=0$ by definition. By symmetry, the above is true also when $m=1$ and $j\neq 1$. Finally, for the case where $j,m\neq 1$, due to the same reasons we get,
\begin{align}
   \s{cov}_{P_{XY}}\p{\calK(A_1,B_j),\calK(A_1,B_m)} &= \bE_{P_{XY}}\pp{\calK(A_1,B_j)\cdot\calK(A_1,B_m)}\\
   & = \bE_{P_{XY}}\pp{\calK(A_1,B_j)\bE[\calK(A_1,B_m)\vert A_1]}\\
   & = 0.
\end{align}
Thus,
\begin{align}
    \s{Var}_{P_{XY}}\p{\sum_{j=1}^n\calK(A_i,B_j)} &=  \s{Var}_{P_{XY}}\p{\calK(A,B)}+(n-1)\cdot \s{Var}_{Q_{XY}}\p{\calK(A,B)}.\label{eqn:VariacnceTermType2}
\end{align}
Next, consider the second term at the right-hand-side of \eqref{eqn:varaCalcType2}. For $i\neq k$, we have,
\begin{align}
\s{cov}_{P_{XY}}\p{\sum_{j=1}^n\calK(A_i,B_j),\sum_{m=1}^n\calK(A_k,B_m)} = \sum_{j,m=1}^n \s{cov}_{P_{XY}}\p{\calK(A_i,B_j),\calK(A_k,B_m)}.\label{eqn:covUnderP}
\end{align}
We split the sum above into two parts; the first correspond to $j=m$, and the second to $j\neq m$. For the former, we have,
\begin{align}
  \sum_{m=1}^n\s{cov}_{P_{XY}}\p{\calK(A_i,B_m),\calK(A_k,B_m)} &=  \s{cov}_{P_{XY}}\p{\calK(A_i,B_i),\calK(A_k,B_i)}\nonumber\\
  &\quad+\s{cov}_{P_{XY}}\p{\calK(A_i,B_k),\calK(A_k,B_k)}\nonumber\\
  &\quad +\sum_{m\neq i,k}\s{cov}_{P_{XY}}\p{\calK(A_i,B_m),\calK(A_k,B_m)}.
\end{align}
 We have,
 \begin{align}
    \s{cov}_{P_{XY}}\p{\calK(A_i,B_i),\calK(A_k,B_i)} &= \bE_{P_{XY}}\p{\calK(A_i,B_i)\cdot\calK(A_k,B_i)} \\
    & = \bE_{P_{XY}}\pp{\calK(A_i,B_i)\cdot\bE_{Q_{XY}}\p{\left.\calK(A_k,B_i)\right| B_i}}\\
    & =0,
 \end{align}
where we have used the fact that $A_k$ and $B_i$ are statistically independent. Similarly, it is clear that $\s{cov}_{P_{XY}}\p{\calK(A_i,B_k),\calK(A_k,B_k)}=0$, and $\s{cov}_{P_{XY}}\p{\calK(A_i,B_m),\calK(A_k,B_m)}=0$. Thus,
\begin{align}
    \sum_{m=1}^n\s{cov}_{P_{XY}}\p{\calK(A_i,B_m),\calK(A_k,B_m)}=0.\label{eqn:covUnderP1}
\end{align}
Now, for $j\neq m$, we have
\begin{align}
    \sum_{j\neq m}^n\s{cov}_{P_{XY}}\p{\calK(A_i,B_j),\calK(A_k,B_m)} &= \s{cov}_{P_{XY}}\p{\calK(A_i,B_i),\calK(A_k,B_k)}\nonumber\\
    &\quad+\s{cov}_{P_{XY}}\p{\calK(A_i,B_k),\calK(A_k,B_i)}\nonumber\\
    &\quad+\sum_{(j,m)\in\calG}\s{cov}_{P_{XY}}\p{\calK(A_i,B_j),\calK(A_k,B_m)},\label{eqn:threeCov}
\end{align}
where $\calG\triangleq\{(j,m)\in[n]\times[n]:j\neq m,(j,m)\neq(i,k),(j,m)\neq(k,i)\}$. Now, since $\{(A_i,B_i)\}_{i=1}^n$ are i.i.d. pairs of random variables, we have,
\begin{align}
    \s{cov}_{P_{XY}}\p{\calK(A_i,B_i),\calK(A_k,B_k)}=0.\label{eqn:ccov1}
\end{align}
On the other hand,
\begin{align}
  \s{cov}_{P_{XY}}\p{\calK(A_i,B_k),\calK(A_k,B_i)} = \bE_{P_{XY}}\pp{\calK(A_i,B_k)\cdot\calK(A_k,B_i)}.\label{eqn:ccov2}
\end{align}
Finally, let us find the covariance term at the right-hand-side of \eqref{eqn:threeCov} for any pair $(j,m)\in\calG$. For any $j\neq i,k$ and $m\neq i,k$, we have
\begin{align}
    \s{cov}_{P_{XY}}\p{\calK(A_i,B_j),\calK(A_k,B_m)}=0,
\end{align}
because the pairs $(A_i,B_j)$ and $(A_k,B_m)$ are statistically independent. Next, for $j=i$ and $m\neq k$, we have
\begin{align}
    \s{cov}_{P_{XY}}\p{\calK(A_i,B_i),\calK(A_k,B_m)} &= 0,
\end{align}
because the pairs $(A_i,B_i)$ and $(A_k,B_m)$ are statistically independent and $\bE_{P_{XY}}[\calK(A_k,B_m)] = \bE_{Q_{XY}}[\calK(A_k,B_m)]=0$ because $A_k$ and $B_m$ are independent. The same is true for $m=k$ and $j\neq i$. For $j=k$ and $m\neq i$, we have
\begin{align}
    \s{cov}_{P_{XY}}\p{\calK(A_i,B_k),\calK(A_k,B_m)} &= \bE_{P_{XY}}\pp{\calK(A_i,B_k)\cdot\calK(A_k,B_m)}\\
    & = \bE_{P_{XY}}\pp{\calK(A_i,B_k)\cdot\bE_{P_{XY}}\p{\left.\calK(A_k,B_m)\right|A_i,B_k,A_k}}\\
    & = \bE_{P_{XY}}\pp{\calK(A_i,B_k)\cdot\bE_{P_{XY}}\p{\left.\calK(A_k,B_m)\right|A_k}}\\
    & = 0,
\end{align}
where in the first equality we have used the fact that $A_k$ and $B_m$ are statistically independent and thus $\bE_{P_{XY}}[\calK(A_k,B_m)] = \bE_{Q_{XY}}[\calK(A_k,B_m)]=0$, the third equality follows because $B_m$ is independent of $A_i$ and $B_k$, and finally, $\bE_{P_{XY}}\p{\left.\calK(A_k,B_m)\right|A_k}=0$ by the definition of $\calK$ and the fact that $A_k$ and $B_m$ are statistically independent. Finally, by symmetry when $m=i$ and $j\neq k$ we also get $\s{cov}_{P_{XY}}\p{\calK(A_i,B_j),\calK(A_k,B_i)}=0$. Combining all the cases above, we finally obtain that
\begin{align}
\sum_{(j,m)\in\calG}\s{cov}_{P_{XY}}\p{\calK(A_i,B_j),\calK(A_k,B_m)}=0,\label{eqn:ccov3}
\end{align}
and thus, from \eqref{eqn:threeCov}, \eqref{eqn:ccov1}, \eqref{eqn:ccov2}, and \eqref{eqn:ccov3}, we have
\begin{align}
    \sum_{j\neq m}^n\s{cov}_{P_{XY}}\p{\calK(A_i,B_j),\calK(A_k,B_m)} = \bE_{P_{XY}}\pp{\calK(A_i,B_k)\cdot\calK(A_k,B_i)}.\label{eqn:covUnderP2}
\end{align}
Accordingly, from \eqref{eqn:covUnderP}, \eqref{eqn:covUnderP1}, and \eqref{eqn:covUnderP2}, we obtain
\begin{align}
   \s{cov}_{P_{XY}}\p{\sum_{j=1}^n\calK(A_i,B_j),\sum_{m=1}^n\calK(A_k,B_m)} =  \bE_{P_{XY}}\pp{\calK(A_i,B_k)\cdot\calK(A_k,B_i)},\label{eqn:LastCovTerm}
\end{align}
which together with \eqref{eqn:VariacnceTerm01}, \eqref{eqn:varaCalcType2}, and \eqref{eqn:VariacnceTermType2}, imply that,
\begin{align}
    \s{Var}_{P_{XY}}\p{\sum_{i,j=1}^n\sum_{\ell=1}^d\calK(X_{i\ell},Y_{j\ell})} &= dn\cdot\s{Var}_{P_{XY}}\p{\calK(A,B)}+dn(n-1)\cdot \s{Var}_{Q_{XY}}\p{\calK(A,B)}\nonumber\\
    &\quad+dn(n-1)\cdot\bE_{P}\pp{\calK(A_1,B_2)\cdot\calK(A_2,B_1)},\label{eqn:VarFinal}
\end{align}
where the last expectation \eqref{eqn:VarFinal} is with respect to $(A_1,B_1)\indep(A_2,B_2)\sim P_{XY}$. Therefore, substituting the last result in \eqref{eqn:TypeIChebyshev3}, we obtain,
\begin{align}
    \pr_{\calH_1}\pp{\phi_{\s{sum}}(\s{X},\s{Y})=0} &\leq \frac{4\cdot\s{Var}_{P_{XY}}\p{\sum_{i,j=1}^n\sum_{\ell=1}^d\calK(X_{i\ell},Y_{j\ell})}}{d^2n^2\cdot\p{d_{\s{KL}}(P_{XY}||Q_{XY})+d_{\s{KL}}(Q_{XY}||P_{XY})}^2}\\
    &\leq \frac{4\cdot\s{Var}_{P_{XY}}\p{\calK(A,B)}}{dn\cdot\p{d_{\s{KL}}(P_{XY}||Q_{XY})+d_{\s{KL}}(Q_{XY}||P_{XY})}^2}\nonumber\\
    &\quad+\frac{4\cdot \s{Var}_{Q_{XY}}\p{\calK(A,B)}}{d\cdot\p{d_{\s{KL}}(P_{XY}||Q_{XY})+d_{\s{KL}}(Q_{XY}||P_{XY})}^2}\nonumber\\
    &\quad+\frac{4\cdot \bE_{P}\pp{\calK(A_1,B_2)\cdot\calK(A_2,B_1)}}{d\cdot\p{d_{\s{KL}}(P_{XY}||Q_{XY})+d_{\s{KL}}(Q_{XY}||P_{XY})}^2}\\
    &\leq \frac{4\cdot\s{Var}_{P_{XY}}\p{\calK(A,B)}}{dn\cdot\p{d_{\s{KL}}(P_{XY}||Q_{XY})+d_{\s{KL}}(Q_{XY}||P_{XY})}^2}\nonumber\\
    &\quad+\frac{8\cdot \s{Var}_{Q_{XY}}\p{\calK(A,B)}}{d\cdot\p{d_{\s{KL}}(P_{XY}||Q_{XY})+d_{\s{KL}}(Q_{XY}||P_{XY})}^2},\label{eqn:TypeIChebyshev4}
\end{align}
where the last inequality follows from the fact that 
\begin{align}
    \bE_{P}\pp{\calK(A_1,B_2)\cdot\calK(A_2,B_1)}&\leq \sqrt{\s{Var}_{P_{XY}}\p{\calK(A_1,B_2)}\s{Var}_{P_{XY}}\p{\calK(A_2,B_1)}}\\
    & = \sqrt{\s{Var}_{Q_{XY}}\p{\calK(A_1,B_2)}\s{Var}_{Q_{XY}}\p{\calK(A_2,B_1)}}\\
    & = \s{Var}_{Q_{XY}}\p{\calK(A,B)},
\end{align}
where the first inequality follows from Cauchy–Schwarz inequality coupled with $\bE_{P}[\calK(A_1,B_2)]=\bE_{P}[\calK(A_2,B_1)]=0$, and the second and third equalities are because $(A_1,B_2)$ and $(A_2,B_1)$ are two pairs of statistically independent random variables, and because $A_1$, $A_2$, $B_1$, and $B_2$, have the same marginal distribution. Finally, based on our upper bounds in \eqref{eqn:TypeIChebyshev2} and \eqref{eqn:TypeIChebyshev4}, we see that the sum test achieves strong detection, i.e., vanishing Type-I and Type-II error probabilities if,
\begin{align}
    \frac{\s{Var}_{Q_{XY}}\p{\calK(A,B)}}{d\cdot\p{d_{\s{KL}}(P_{XY}||Q_{XY})+d_{\s{KL}}(Q_{XY}||P_{XY})}^2}\to0,
\end{align}
as $d\to\infty$, which concludes the proof.

\subsection{Proof of Theorem~\ref{thm:IndSum}}

Let 
\begin{align}
  \calQ_{d}&\triangleq  \pr_{Q_{XY}^{\otimes d}}\pp{\frac{1}{d}\sum_{\ell=1}^d\log\frac{P_{XY}(A_{\ell},B_{\ell})}{Q_{XY}(A_{\ell},B_{\ell})}\geq\tau_{\s{count}}},\\
  \calP_{d}&\triangleq  \pr_{P_{XY}^{\otimes d}}\pp{\frac{1}{d}\sum_{\ell=1}^d\log\frac{P_{XY}(A_{\ell},B_{\ell})}{Q_{XY}(A_{\ell},B_{\ell})}\geq\tau_{\s{count}}}.
\end{align}
Consider the test in \eqref{eqn:testcount}. We start by bounding the Type-I error probability. Markov's inequality implies that
\begin{align}
    \pr_{\calH_0}\p{\phi_{\s{count}}=1} &= \pr_{\calH_0}\p{\sum_{i,j=1}^n\Ind\ppp{\frac{1}{d}\log\frac{P^{\otimes d}_{XY}(X_{i},Y_{j})}{Q^{\otimes d}_{XY}(X_{i},Y_{j})}\geq\tau_{\s{count}}}\geq \frac{1}{2} n\calP_{d}}\\
    &\leq \frac{2n\calQ_{d}}{\calP_{d}}.\label{eqn:type1count}
\end{align}
On the other hand, we bound the Type-II error probability as follows. Under $\calH_1$, since our proposed test is invariant to reordering of $\s{X}$ and $\s{Y}$, we may assume without loss of generality that the latent permutation is the identity one, i.e., $\sigma=\s{Id}$. Then, Chebyshev's inequality implies that
\begin{align}
     \pr_{\calH_1}\p{\phi_{\s{count}}=0} &= \pr_{\calH_1}\p{\sum_{i,j=1}^n\Ind\ppp{\frac{1}{d}\log\frac{P^{\otimes d}_{XY}(X_{i},Y_{j})}{Q^{\otimes d}_{XY}(X_{i},Y_{j})}\geq\tau_{\s{count}}}< \frac{1}{2} n\calP_{d}}\\
     &\leq \pr_{\calH_1}\p{\sum_{i=1}^n\Ind\ppp{\frac{1}{d}\log\frac{P^{\otimes d}_{XY}(X_{i},Y_{i})}{Q^{\otimes d}_{XY}(X_{i},Y_{i})}\geq\tau_{\s{count}}}< \frac{1}{2} n\calP_{d}}\\
&\leq\frac{4\cdot\s{Var}_{P_{XY}^{\otimes d}}\p{\sum_{i=1}^n\Ind\ppp{\frac{1}{d}\log\frac{P^{\otimes d}_{XY}(X_{i},Y_{i})}{Q^{\otimes d}_{XY}(X_{i},Y_{i})}\geq\tau_{\s{count}}}}}{n^2\calP^2_{d}},
\end{align}
where $\s{Var}_{P_{XY}^{\otimes d}}$ denotes the variance with respect to $P_{XY}^{\otimes d}$. Noticing that 
\begin{align}
    &\s{Var}_{P_{XY}^{\otimes d}}\p{\sum_{i=1}^n\Ind\ppp{\frac{1}{d}\log\frac{P^{\otimes d}_{XY}(X_{i},Y_{i})}{Q^{\otimes d}_{XY}(X_{i},Y_{i})}\geq\tau_{\s{count}}}} \nonumber\\
    &\quad\quad= \sum_{i=1}^n\s{Var}_{P_{XY}^{\otimes d}}\p{\Ind\ppp{\frac{1}{d}\log\frac{P^{\otimes d}_{XY}(X_{i},Y_{i})}{Q^{\otimes d}_{XY}(X_{i},Y_{i})}\geq\tau_{\s{count}}}}\\
    &\quad\quad = n\calP_{d}(1-\calP_{d}),
\end{align}
we finally obtain,
\begin{align}
     \pr_{\calH_1}\p{\phi_{\s{count}}=0} &\leq\frac{4(1-\calP_{d})}{n\calP_{d}}\leq\frac{4}{n\calP_{d}}.\label{eqn:type2count}
\end{align}
Standard Chernoff bounds on these tails yield that if the threshold $\tau_{\s{count}}$ is such that $\tau_{\s{count}}\in(-d_{\s{KL}}(Q_{XY}||P_{XY}),d_{\s{KL}}(P_{XY}||Q_{XY}))$, then
\begin{align}
    \calQ_d&\leq \exp\pp{-d\cdot E_Q(\tau_{\s{count}})},\\
    \calP_d&\geq 1-\exp\pp{-d\cdot E_P(\tau_{\s{count}})}.
\end{align}
Thus, we obtain
\begin{align}
    \pr_{\calH_0}\p{\phi_{\s{count}}=1}&\leq 2n\cdot\frac{\exp\pp{-d\cdot E_Q(\tau_{\s{count}})}}{1-\exp\pp{-d\cdot E_P(\tau_{\s{count}})}},
\end{align}
and
\begin{align}
    \pr_{\calH_1}\p{\phi_{\s{count}}=0}&\leq \frac{4}{n\cdot\p{1-\exp\pp{-d\cdot E_P(\tau_{\s{count}})}}}.
\end{align}
For strong detection, we see that the Type-II error probability converges to zero as $n\to\infty$, if $1-e^{-d\cdot E_P(\tau_{\s{count}})}=\omega(n^{-1})$, which is equivalent to $E_P(\tau_{\s{count}}) = \omega(n^{-1}d^{-1})$. Under this condition we see that the Type-I error probability converges to zero if $ne^{-d\cdot E_Q(\tau_{\s{count}})} = o(1)$, which is equivalent to $E_Q(\tau_{\s{count}}) = \omega(\log n^{1/d})$.

\section{Conclusions and Outlook}

In this paper, we analyzed the problem of detecting the correlation of dependent databases with general distributions. To that end, we derived both lower and upper bounds on the phase transition boundaries at which detection is possible or impossible. Our bounds are tight in the special cases considered in the literature. We hope our work has opened more doors than it closes. Let us mention briefly a few interesting questions for future research:
    \begin{enumerate}
        \item In this paper we assume that both generative distributions $P_{XY}$ and $Q_{XY}$ are known to the learner. In practice, it would be interesting to devise universal algorithms which are independent of $P_{XY}$ and $Q_{XY}$, but at the same time achieve comparable results to the case where those distributions are known. Furthermore, it would be interesting to devise adaptive probing detection algorithms for scenarios where one is able to probe only part of the data, while the complete databases are not given/expensive to obtain. Finally, it would be interesting to construct differentially private algorithms.
        \item We focused on the case where the data is i.i.d. given the permutation. However, in practice, it is reasonable that there will be some structured dependency over, for example, the feature space. As so, it would be interesting to generalize our results to a model which captures such a dependency. Furthermore, we assumed that the permutation is drawn uniformly over $\mathbb{S}_n$. What if there is a different prior, or, a structured set from which the permutation is drawn? As an example, consider the case where there are groups of users with similar features, and the permutation is applied per each group.
        \item We investigated the question of ``independent vs. dependent" databases. It would be interesting to analyze the related, but easier, question of ``uncorrelated vs. correlated"? Obviously, two databases can be dependent but uncorrelated. Namely, this is a question about first order dependency.
        \item It would be interesting to check whether the techniques in this paper can be used to solve the graph matching detection problem in \cite{wu2020testing}, for general weight distributions.
    \end{enumerate}

\bibliographystyle{alpha}
\bibliography{bibfile}

\end{document}